\documentclass[11pt]{article}
\pdfoutput=1
\usepackage{wrapfig}
\usepackage[utf8]{inputenc} 
\usepackage[T1]{fontenc}    
\usepackage[colorlinks,
            linkcolor=red,
            anchorcolor=blue,
            citecolor=blue
            ]{hyperref}
\usepackage{url}            
\usepackage{booktabs}       
\usepackage{amsfonts}       
\usepackage{nicefrac}       
\usepackage{microtype}      
\usepackage{xcolor,colortbl}         
\usepackage{wrapfig}
\usepackage{caption}
\usepackage{tabularx}

\usepackage{setspace}
\usepackage{fullpage}
\RequirePackage{amsmath}
\RequirePackage{amssymb}
\RequirePackage{amsthm}
\RequirePackage{bm} 
\RequirePackage{url}
\usepackage{multirow}
\usepackage{natbib}
\usepackage{graphicx}
\usepackage{makecell}
\usepackage{booktabs}
\usepackage{array}
\usepackage{url}
\usepackage{dsfont}
\usepackage[shortlabels,inline]{enumitem}
\newlist{parenum}{enumerate}{1}
\setlist[parenum,1]{label=\textbf{(\arabic*)}, wide, labelwidth=!, labelindent=0pt, topsep=0pt, itemsep=-1ex, partopsep=0ex, parsep=1ex}
\usepackage{paralist}

\usepackage{subcaption}
\usepackage{caption}
\usepackage[ruled, linesnumbered, noend, vlined]{algorithm2e}
\usepackage{adjustbox}
\usepackage{multirow}
\usepackage{multicol}
\usepackage{colortbl}
\usepackage{setspace}
\usepackage[para,online,flushleft]{threeparttable}
\newcommand{\wcm}[1]{\textbf{#1}}
\newcommand{\scm}[1]{\textbf{#1}\quad}
\newtheorem{assumption}{Assumption}

\usepackage{amsmath,amsfonts,bm}

\def\eqref#1{equation~\ref{#1}}

\def\1{\bm{1}}

\def\va{{\bm{a}}}

\DeclareMathAlphabet{\mathsfit}{\encodingdefault}{\sfdefault}{m}{sl}
\SetMathAlphabet{\mathsfit}{bold}{\encodingdefault}{\sfdefault}{bx}{n}

\def\sR{{\mathbb{R}}}

\def\sZ{{\mathbb{Z}}}

\newcommand{\diag}[1]{\mathsf{diag}(#1)}
\newcommand{\wt}[1]{\widetilde{#1}}
\newcommand{\tran}{^{\mkern-1.5mu\mathsf{T}}}
\newcommand{\vones}{\mathbbm{1}}
\newcommand{\na}{^\mathtt{a}}
\newcommand{\lc}{^\mathtt{c}}
\newcommand{\egv}{\mathbf{u}}
\newcommand{\cv}{\mathbf{v}}

\newcommand{\concat}{\bigparallel}

\newcommand{\nonlinear}{\sigma}

\newcommand{\fC}{\mathfrak{C}}

\newcommand{\itand}{~~{\footnotesize\textsf{and}}~~}

\usepackage{amsmath}
\usepackage{amsthm}
\usepackage{amsfonts}
\usepackage{mathtools}
\usepackage{bm}
\usepackage{bbm}
\usepackage{stmaryrd}
\usepackage{mathrsfs}
\renewcommand{\emph}[1]{\textit{#1}}
\usepackage[all]{nowidow}
\usepackage{adjustbox}
\usepackage{threeparttable}
\usepackage{etoolbox}
\appto\TPTnoteSettings{\footnotesize}
\usepackage{cellspace}
\usepackage{booktabs}
\usepackage{multirow}
\usepackage{setspace}
\usepackage{subcaption}
\usepackage{tikz}
\usepackage{pgfplots}
\usetikzlibrary{calc,shapes}
\usetikzlibrary{decorations.pathmorphing} 
\usetikzlibrary{fit}					
\usetikzlibrary{backgrounds}	
\usetikzlibrary{pgfplots.groupplots}
\usepackage{xargs}
\usepackage[capitalize,noabbrev]{cleveref}
\newtheoremstyle{break}
  {}
  {}
  {\itshape}
  {}
  {\bfseries}
  {}
  {\newline}
  {}
\newtheorem{theorem}{Theorem}
\newtheorem{proposition}[theorem]{Proposition}
\newtheorem{lemma}[theorem]{Lemma}
\newenvironment{prevproof}[1]{\noindent {\em {Proof of \cref{#1}:}}}{\hfill $\square$\vskip \belowdisplayskip}
\newtheorem{innercustomgeneric}{\customgenericname}
\providecommand{\customgenericname}{}
\newcommand{\newcustomtheorem}[2]{%
  \newenvironment{#1}[1]
  {%
   \renewcommand\customgenericname{#2}%
   \renewcommand\theinnercustomgeneric{##1}%
   \innercustomgeneric
  }
  {\endinnercustomgeneric}
}
\newcustomtheorem{customtheorem}{Theorem}
\newcustomtheorem{customlemma}{Lemma}
\newcustomtheorem{customcorollary}{Corollary}

\usepackage{colortbl}
\pgfplotsset{compat=1.18}
\allowdisplaybreaks
\renewcommand{\emph}[1]{\textit{#1}}

\title{\huge Spectral Greedy Coresets for Graph Neural Networks}

\author{
    Mucong Ding\thanks{University of Maryland, College Park; e-mail: {\tt mcding@umd.edu}}
    \and
    Yinhan He\thanks{University of Virginia, Charlottesville; e-mail: {\tt nee7ne@virginia.edu}}
    \and
    Jundong Li\thanks{University of Virginia, Charlottesville}  
    \and
    Furong Huang\thanks{University of Maryland, College Park}
}

\begin{document}
\date{}
\maketitle

\begin{abstract}
The ubiquity of large-scale graphs in node-classification tasks significantly hinders the real-world applications of Graph Neural Networks (GNNs).
Node sampling, graph coarsening, and dataset condensation are effective strategies for enhancing data efficiency. However, owing to the interdependence of graph nodes, coreset selection, which selects subsets of the data examples, has not been successfully applied to speed up GNN training on large graphs, warranting special treatment.
This paper studies graph coresets for GNNs and avoids the interdependence issue by selecting ego-graphs (i.e., neighborhood subgraphs around a node) based on their spectral embeddings.
We decompose the coreset selection problem for GNNs into two phases, a coarse selection of widely spread ego graphs and a refined selection to diversify their topologies.
We design a greedy algorithm that approximately optimizes both objectives.
Our spectral greedy graph coreset (SGGC) scales to graphs with millions of nodes, obviates the need for model pre-training, and is applicable to low-homophily graphs.
Extensive experiments on ten datasets demonstrate that SGGC outperforms other coreset methods by a wide margin, generalizes well across GNN architectures, and is much faster than graph condensation.
\end{abstract}

\section{Introduction}
\label{sec:intro}

\emph{Graph neural networks} (GNNs) have achieved significant success in addressing various graph-related tasks, such as node classification and link prediction~\citep{hamilton2020graph}.
Nonetheless, the widespread presence of large-scale graphs in real-world scenarios, including social, informational, and biological networks, presents substantial computational challenges for GNN training, given the frequent scaling of these graphs to millions of nodes and edges.
The cost of training a single model is considerable and escalates further when multiple training iterations are required, for example, to validate architectural designs and hyperparameter selections~\citep{elsken2019neural}.
To tackle the above issues, we adopt a natural \emph{data-efficiency} approach --- simplifying the given graph data appropriately, with the goal of saving training time.
In particular, we ask the following question: \emph{how can we appropriately simplify graphs while preserving the performance of GNNs}?

A simple yet effective solution to simplify a dataset is \emph{coreset selection}, despite other methods such as graph sparsification, graph coarsening, and graph condensation reviewed in the related work~\cref{sec:related}.
Typically, the coreset selection approach~\citep{toneva2018empirical,paul2021deep} finds subsets of data examples that are important for training based on certain heuristic criteria.
The generalization of coreset selection to graph node/edge classification problems is then to find the important ``subsets'' of the given graph, e.g., nodes, edges, and subgraphs.
This challenge arises from graph nodes' interdependence and GNNs' non-linearities.
We focus on node classification in this paper as it is among the important learning tasks on graphs and is still largely overlooked.

\begin{figure*}[!ht]
\centering
\includegraphics[width=\textwidth]{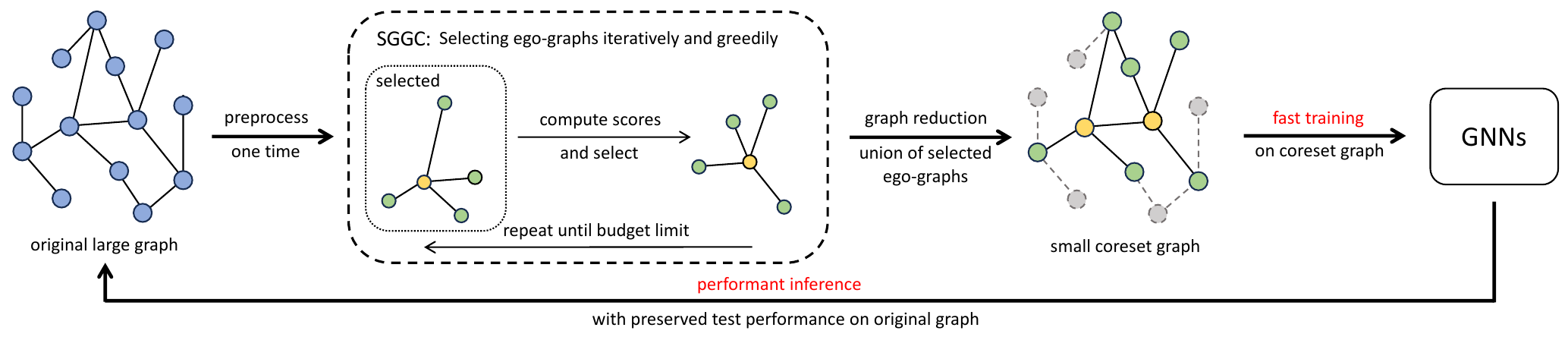}
\caption{Overview of \textit{spectral greedy graph coresets (SGGC)} for efficient GNN training. SGGC processes a large graph to iteratively select ego-graphs. The assembled coreset graph facilitates fast GNN training while maintaining test performance on the original graph.}
\label{fig:diagram}
\end{figure*}

In this paper, we find a new approach to formulate graph coreset selection for GNNs.
It avoids GNN's interdependent nodes and non-linear activation issues by selecting ego-graphs, i.e., the subgraph induced by the neighborhood around a center node, based on their node embeddings in the graph-spectral domain.
Our ego-graph-based coreset is inspired by two observations.
\textbf{(1)} We find that most GNNs applied to large graphs follow the nearest-neighbor message-passing update rule and have ego-graph-like receptive fields.
Thus, by selecting the ego-graphs (which is equivalent to selecting their center nodes), we avoid the problem of finding subsets of nodes and edges independently, which typically leads to complex combinatorial optimization; see~\cref{sec:prelim}.
\textbf{(2)} Moreover, we identify that when expressing the node embeddings in the graph-spectral domain, the non-linear spectral embeddings of GNNs on ego-graphs are ``smooth'' functions of the center node, i.e., nearby nodes are likely to have similar spectral embeddings on their corresponding ego-graphs~\citep{balcilar2021analyzing}, which we will theoretically justify under certain assumptions in~\cref{sec:theory}.

Using \textbf{(1)} and \textbf{(2)}, we propose approximating the GNN's spectral embedding using a subset of ego-graphs.
\emph{To approximate the spectral embedding with fewer ego-graphs} (which one-to-one correspond to their center nodes), one should select center nodes that are far from each other since nearby ones are likely to have similar embeddings, thus, being less informative.
We derive an upper bound on the coreset objective independent of the spectral embedding.
This enables us to find the coreset of center nodes without evaluating the GNN's spectral embeddings on any ego-graph.
With the coreset objective substituted by the upper-bound, we adopt the greedy iterative geodesic ascent (GIGA) approach~\citep{campbell2018bayesian,vahidian2020coresets} to obtain the coresets.

The above procedure of selecting distant ego-graphs is sufficient to approximate the whole graph's spectral embedding well.
However, the selected center nodes do not necessarily approximate the node classification loss well, and the topological information of ego-graphs is not considered.
To approximate the GNN training loss, we propose to refine the coreset selection by filtering out some of the selected ego-graphs whose topologies are not distinctive enough. 
Since the transformation from the spectral to the spatial domain is a linear operation depending on the graph topology, the approximated spatial embeddings of ego-graphs will differ more when they have more distinctive topologies.
Hence, we exclude ego-graphs with non-distinctive spatial embeddings to enhance efficiency.
This is solved by the submodular coreset algorithm~\citep{iyer2021submodular,kothawade2022prism} using greedy submodular maximization~\citep{mirzasoleiman2020coresets}.

As a result, we decompose the ego-graph selection into two stages: a coarse selection of widely spread ego-graphs that approximate the whole graph's spectral embedding (as detailed in \cref{eq:graph-average-coreset}) and a refined selection to approximate the node classification loss with improved sample efficiency (as detailed in \cref{eq:linear-classification-coreset}).
Specifically, the first stage (which is solved via GIGA) \emph{extrapolates} the graph to find distant center nodes over the original graph, and the second stage (which is solved via submodular maximization) \emph{exploits} the candidate center nodes and keeps the most informative ones based on their topologies. 
We call this two-stage algorithm \textit{spectral greedy graph coresets (SGGC)}.
Our SGGC compresses node attributes of selected ego-graphs using low-rank approximations, maintaining efficient storage without compromising GNN performance, as shown in~\cref{sec:experiments}.

A visualization overview of the SGGC approach, from preprocessing a large graph to forming a coreset graph for fast GNN training, is provided in~\cref{fig:diagram}.
SGGC scales to large graphs, needs no pre-training, and performs well on both high- and low-homophily graphs.
SGGC surpasses other coreset methods in our experiments on ten graph datasets.
We show that the combined algorithm is better than applying either algorithm (GIGA or submodular maximization) individually. 
Moreover, SGGC matches graph condensation's performance~\citep{jin2021graph}, but is significantly faster and better generalizes across GNN architectures.

Our \textbf{contributions} are summarized as follows:
\begin{parenum}
\item \textbf{Novel Graph Coreset Selection Method for GNNs:} We propose a novel graph coreset selection method for GNNs, leveraging ego-graphs to address computational challenges in training on large-scale graphs, which simplifies the data while preserving GNN performance.
\item \textbf{Two-Stage Spectral Greedy Graph Coresets Algorithm:} Our approach introduces a two-stage algorithm, SGGC, that efficiently reduces graph size for GNN training by first broadly selecting ego-graphs and then refining this selection based on topological distinctiveness.
\item \textbf{Theoretical Foundation:} We provide a theoretical justification for our method, showing that ego-graph-based coresets can approximate GNN spectral embeddings with smooth functions, enabling a more effective simplification of graphs.
\item \textbf{Empirical Validation Across Diverse Graphs:} SGGC demonstrates superior performance in experiments on ten graph datasets, offering significant improvements in efficiency, scalability, and generalizability over existing coreset and condensation methods.
\end{parenum}
\section{Problem: Graph Coresets for GNNs}
\label{sec:prelim}

\begin{figure*}[htbp!]
    \centering
    \begin{minipage}[t]{0.32\textwidth}
        \centering
        \includegraphics[width=.99\linewidth,trim={5pt 5pt 5pt 5pt},clip]{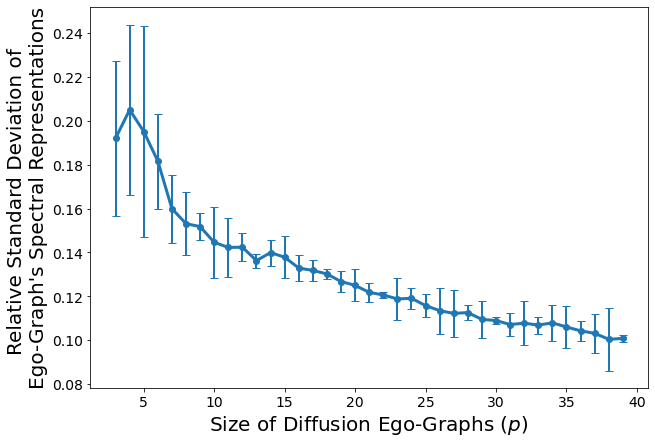}
        \caption{Relative standard deviation of spectral embeddings on ego-graphs $\boldsymbol{Z}_i$ across all the nodes vs. the ego-graph size $p$; see~\cref{assump:bounded-spectral-variance}.}
        \label{fig:spectral-concentration}
    \end{minipage}
    \hfill
    \begin{minipage}[t]{0.32\textwidth}
        \centering
        \includegraphics[width=.90\linewidth,trim={5pt 0pt 5pt 5pt},clip]{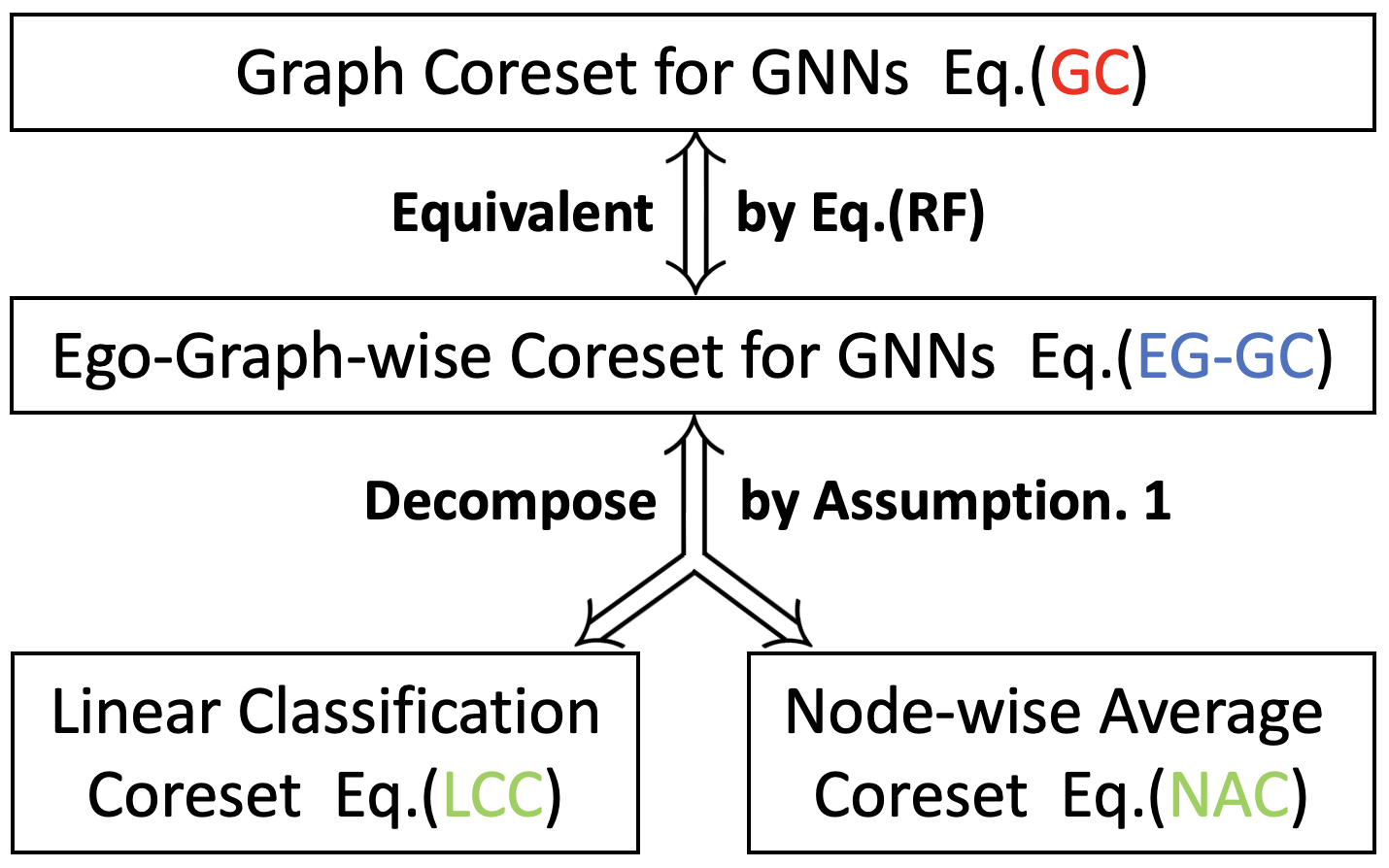}
        \caption{Conceptual diagram showing the theoretical analysis formulating the spectral greedy graph coresets (SGGC).}
        \label{fig:concept}
    \end{minipage}
    \hfill
    \begin{minipage}[t]{0.32\textwidth}
        \centering
        \includegraphics[width=.96\linewidth,trim={5pt 5pt 5pt 5pt},clip]{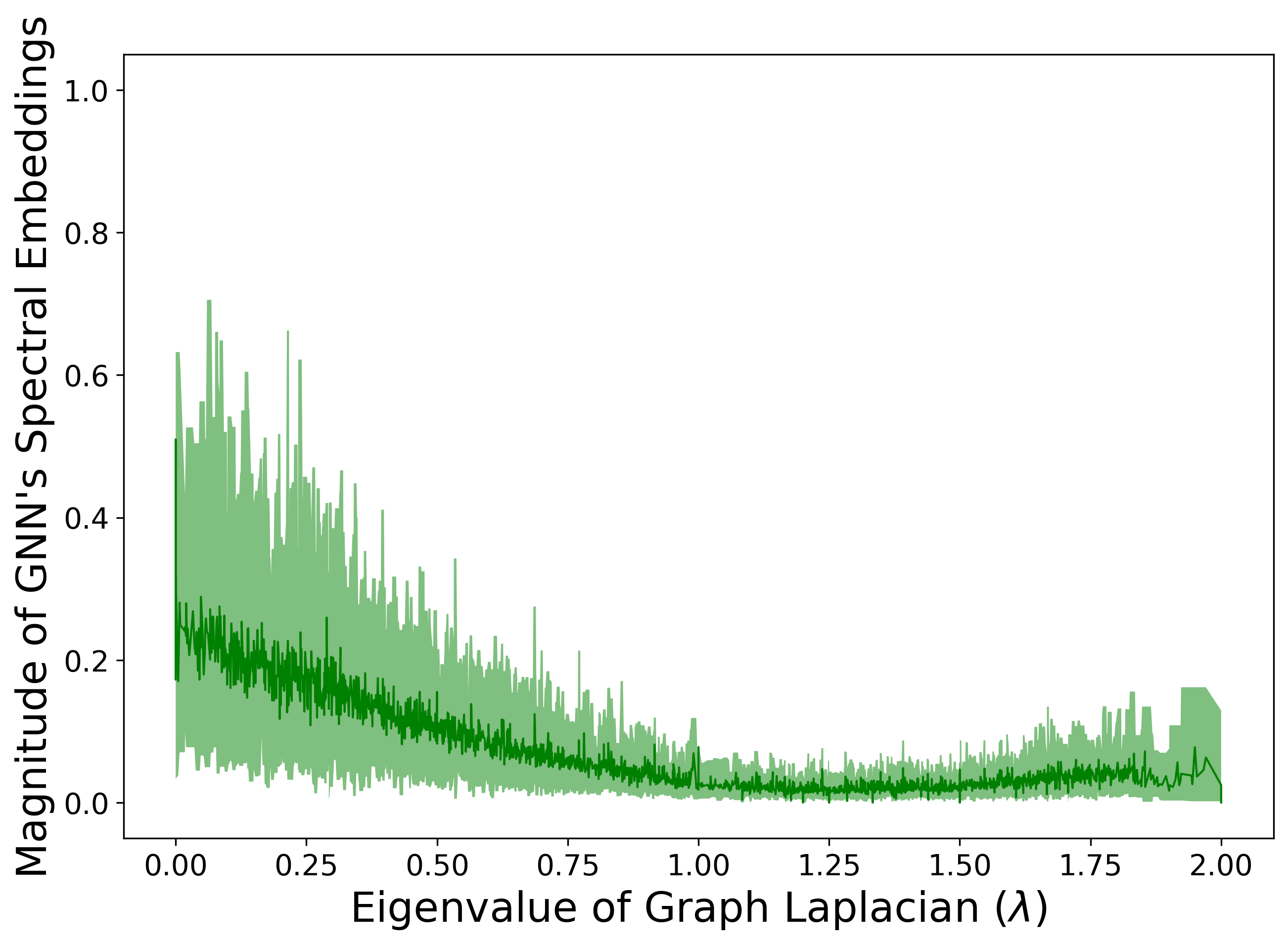}
        \caption{Spectral response of 2-layer GCNs on Cora. The spectral response corresponding to eigenvalue $\lambda_i$ is defined as 
        $\|[U\tran f_\theta(A,X)]_{i,:}\|/\|[U\tran X]_{i,:}\|$.}
        \label{fig:gnn-spectral}
    \end{minipage}
\end{figure*}

We start by defining the downstream task for graph coresets and node classification with graph neural networks.
For a matrix $M$, we denote its $(i,j)$-th entry, $i$-th row, $j$-th column by $M_{i,j}$, $M_{i,:}$, and $M_{:,j}$, respectively.

\wcm{Node classification on a graph} considers that we are given an (undirected) graph $G=(V=[n],~E\subset[n]\times[n])$ with (symmetric) adjacency matrix $A\in\{0,1\}^{n\times n}$, node features $X\in\mathbb{R}^{n\times d}$, node class labels $\mathbf{y}\in[K]^n$, and mutually disjoint node-splits $V_{\text{train}}\bigcup V_{\text{val}}\bigcup V_{\text{test}}=[n]$, where we assume the training split consists of the first $n_t$ nodes, i.e., $V_{\text{train}}=[n_t]$.
Using a \emph{graph neural network} (GNN) $f_{\theta}: \mathbb{R}_{\geq0}^{n\times n}\times \mathbb{R}^{n\times d}\to \mathbb{R}^{n\times K}$, where $\theta\in\Theta$ denotes the parameters, we aim to find $\theta^*=\arg\min_{\theta}\mathcal{L}(\theta)$, where the training loss is $\mathcal{L}(\theta)\coloneqq \frac{1}{n_t}\sum_{i\in [n_t]}\ell([f_{\theta,\lambda}(A,X)]_{i,:},~y_i)$.
Here $Z_i\coloneqq[f_{\theta,\lambda}(A,X)]_{i,:}$ is the output embedding for the $i$-th node.
The node-wise loss function is $\ell(Z_i,y_i)\coloneqq\text{CrossEntropy}(\text{softmax}(Z_i), y_i)$. The loss $\mathcal{L}(\theta)$ defined above is under the \emph{transductive} setting, which can be generalized to the \emph{inductive} setting by assuming only $\{A_{ij}\mid i,j\in[n_t]\}$ and $\{X_i\mid i\in[n_t]\}$ are used during training.

\wcm{Graph neural networks} (GNNs) can often be interpreted as iterative convolution over nodes (i.e., \emph{message passing})~\citep{ding2021vq}, where given inputs $X^{(0)}=X$,
\begin{equation}
\label{eq:gnn}
    X^{(l+1)}=\sigma(C_{\alpha^{(l)}}(A)X^{(l)}W^{(l)})\quad\forall l\in[L],
\tag{GNN}
\end{equation}
and $X^{(L)}=f_{\theta}(A,X)$.
Here, $C_{\alpha^{(l)}}(A)$ is the convolution matrix which is (possibly) parametrized by $\alpha^{(l)}$, $W^{(l)}$ is the learnable linear weights, and $\sigma(\cdot)$ denotes the non-linearity. See~\cref{apd:gnns} for more details on GNNs.

\wcm{Graph coresets for GNNs} 
\textbf{Graph coresets for GNNs} seek to select a subset of training nodes $V_w\subset V_{train}=[n_t]$ with size $|V_w|\leq c \ll n_t$ along with a corresponding set of sample weights such that the training loss $\mathcal{L}(\theta)$ is approximated for any $\theta\in\Theta$. Let $w\in\mathbb{R}^{n_t}_{\geq0}$ be the vector of non-negative weights, we require the number of non-zero weights $\|w\|_0\coloneqq\sum_{i\in[n_t]}\vones[w_i>0]\leq c$, and hence the search space is $\mathcal{W}\coloneqq\{w\in\mathbb{R}^{n_{t}}_{\geq0}\mid\|w\|_0\leq c\}$. The objective of graph coreset selection is
\begin{equation}
\label{eq:graph-coreset}
    \min_{w\in\mathcal{W}}\max_{\theta\in\Theta}\Big|\sum_{i\in [n_{t}]}w_i\cdot\ell\big([f_\theta(A,X)]_i,~y_i\big)-\mathcal{L}(\theta)\Big|,
\tag{\textcolor{red}{GC}}
\end{equation}
which minimizes the worst-case error for all $\theta$.
However, the above formulation \emph{provides nearly no data reduction} in terms of its size, since both the graph adjacency matrix $A$ and node feature matrix $X$ are still needed to compute the full-graph embeddings $f_\theta(A,X)$ in the coreset loss in~\cref{eq:graph-coreset}.

Since the goal of graph coresets is not only to reduce the number of labels, but more importantly, the data size, we should formalize how the subsets of $A$ and $X$ are selected in~\cref{eq:graph-coreset} to make the objective practical.
A natural convention, considered by related literature including graph condensation~\citep{jin2021graph}, is further \emph{assuming} that only the node features of the selected nodes, $X_w=\{X_{i,:}\mid i\in V_w\}$, and the subgraph induced by the selected nodes, $A_w=\{A_{i,j}\mid i,j\in V_w\}$ are kept.
Under this convention, the central problem of graph coresets changes to selecting the labeled nodes, as well as their node features and adjacencies, which we call as \wcm{node-wise graph coresets},
\begin{equation}
\label{eq:node-coreset}
    \min_{w\in\mathcal{W}}\max_{\theta\in\Theta}\Big|\sum_{i\in [n_{t}]}w_i\cdot\ell\big([f_\theta(A_w,X_w)]_i,~y_i\big)-\mathcal{L}(\theta)\Big|.
\tag{\textcolor{red}{N-GC}}
\end{equation}
However, since $A_w$ and $X_w$ are complex discrete functions of $w$, the above formulation leads to a complex combinatorial optimization, which we still need to learn how to solve.
Moreover, posing $A_w$ and $X_w$ to be entirely determined by $w$ leads to sub-optimality.
For example, it is likely that there exists another set of sample weights $w'\in\mathcal{W}$ such that using $A_{w'}$ and $X_{w'}$ in~\cref{eq:node-coreset} results in a smaller error.

\scm{Alternative formulation of graph coresets.} The critical difference between~\cref{eq:node-coreset} and a typical coreset problem like~\cref{eq:graph-coreset} is that the node-pair-wise relation encoded in the adjacency matrix $A$ forbids us to select nodes as independent samples.
In this paper, we consider another formation of graph coresets, to avoid the non-independent issue.
The idea is to use the property that most GNNs (especially those applied to large graphs) are ``\emph{local functions}'' on the graph, i.e., the output embedding $Z_i=[f_{\theta,\lambda}(A,X)]_{i,:}$ of the $i$-th node may only depend on the close-by nodes $\{j\in[n]\mid d(i,j)<D\}$ where $d(i,j)$ denotes the shortest-path distance.
Without loss of generality, we consider nearest-neighbor message passing (including GCN, GAT, and GIN), whose convolution/message-passing weight is non-zero $C_{i,j}\neq 0$ if and only if $i=j$ or $A_{i,j}=1$.
More specifically, we define the \wcm{receptive field} of a node $i$ for an $L$-layer GNN (\cref{eq:gnn}) as a set of nodes $V^L_i$ whose features $\{X_{j,:}\mid j\in V^L_i\}$ determines $Z_i$.
For nearest-neighbor message passing, it is easy to see $V^1_i=\{i\}\cup\{j\in[n]\mid A_{i,j}=1\}=\{j\in[n]\mid d(i,j)\leq 1\}$.
Then by induction, for $L$-layer GNNs, the respective filed of node $i$ is $V^L_i=\{j\in[n]\mid d(i,j)\leq L\}$, which is exactly its depth-$L$ \emph{ego-graph}.
Here, we assume the GNN is $L$-layered, and the depth-$L$ \wcm{ego-graph} of node $i$, denoted by $G_i$, is defined as the induced subgraph of nodes within distance $L$.
The above characterization of the ``\emph{local property}'' of GNNs leads to the following equation,
\begin{equation}
\label{eq:receptive-field}
\big[f_\theta(A,X)\big]_{i,:} = \big[f_\theta(A_{G_i},X_{G_i})\big]_{1,:}\quad\forall~i\in[n],
\tag{RF}
\end{equation}
where $A_{G_i}$ and $X_{G_i}$ denote the adjacencies and node features in the ego-graph $G_i$, respectively, where we always \emph{re-number} the center node $i$ in $G_i$ as the first node. 

\wcm{Ego-graph-wise graph coreset} can then be formulated by substituting~\cref{eq:receptive-field} into~\cref{eq:graph-coreset},
\begin{equation}
\label{eq:ego-spacial-coreset}
\min\limits_{w\in\mathcal{W}}\max\limits_{\theta\in\Theta}\Big|\sum_{i\in [n_{t}]}w_i\cdot\ell\big([f_\theta(A_{G_i},X_{G_i})]_{1,:},~y_i\big)-\mathcal{L}(\theta)\Big|.
\tag{\textcolor{blue}{EG-GC}}
\end{equation}
Compared with node-wise graph coreset (\cref{eq:node-coreset}), ego-graph-wise selection has the following advantages:
\begin{enumerate*}[label=(\arabic*)]
    \item it avoids the non-independence issue as we are now selecting ego-graphs independently, i.e., whether $G_j$ ($j\neq i$) is selected will not affect the embedding $[f_\theta(A_{G_i}, X_{G_i})]_{1,:}$ of node $i$;
    \item it is equivalent to the original objective (\cref{eq:graph-coreset}) which ensures optimality; and
    \item although the adjacencies and node features in the union of selected ego-graphs $\bigcup_{i\in V_w}G_i$ are kept and their size could be $O(d_{\max}^L)$ times of the node-wise selected data (where $d_{\max}$ is the largest node degree), we find that we can highly compress the ego-graph node features via principal component analysis (PCA), depending on how far away the nodes are from the selected center nodes $V_w$, which eventually leads to comparable data size reduction. See~\cref{fig:ego-compression} and~\cref{apd:proofs} for details. 
\end{enumerate*}

\section{Spectral Greedy Graph Coresets}
\label{sec:theory}

Despite the numerous benefits of selecting ego-graph subsets, addressing~\cref{eq:ego-spacial-coreset} remains a \emph{challenging} task due to the highly non-linear nature of the GNN $f_\theta$, and it is \emph{costly} because the evaluation of $A_{G_i}$ and $X_{G_i}$ necessitates locating the ego-graph $G_i$, a process that incurs a time complexity of $O(d_{\max}^L)$.
This section introduces a method that is both efficient and effective for solving~\cref{eq:ego-spacial-coreset}, designed to \emph{circumvent the non-linearities} present in $f_\theta$ without the need for direct evaluation of $A_{G_i}$ and $X_{G_i}$ for \emph{any} node.
The key idea involves reformulating~\cref{eq:ego-spacial-coreset} within the \emph{graph spectral domain}.

\wcm{Graph spectral domain} refers to the eigenspace of the graph Laplacian, which also encompasses the corresponding spectral feature space.
Consider the symmetrically normalized Laplacian $L=I_n-D^{-1/2}AD^{-1/2}$ where $D$ is the diagonal degree matrix.
By applying eigendecomposition, we obtain $L=U\diag{\lambda_1,\ldots,\lambda_n}U\tran$, where the eigenvalues satisfy $0\leq\lambda_1\leq\cdots\leq\lambda_n\leq2$, and each column $\egv_i=U_{:,i}$ represents an eigenvector.
Features or embeddings can be transformed into the spectral domain by left-multiplying with $U\tran$, for instance, $U\tran X$. Here, the $i$-th row $[U\tran X]_{i,:}$ corresponds to the features associated with eigenvalue $\lambda_i$.

Likewise, for each ego-graph $G_i$, it is possible to identify the \emph{spectral representation of ego-graph embeddings}, represented as $\wt{Z}i=U{G_i}\tran f_\theta(A_{G_i},X_{G_i})$.
To simplify our analysis and ensure that $\wt{Z}_i$ associated with different eigenvalues have identical dimensions, we introduce a nuanced concept of ego-graphs, termed \emph{diffusion ego-graphs}.
Consider the diffusion matrix $P=\frac12I_n+\frac12D^{-1}A$, characterized as right stochastic --- meaning each row sums to $1$, and it models a lazy-random walk across the graph.
The matrix $P$ can be diagonalized simultaneously with $L$, producing eigenvalues in the order $1\geq1-\frac12\lambda_1\geq\cdots\geq1-\frac12\lambda_n\geq0$.
The diffusion ego-graph $\wt{G}_i$ of node $i$ s defined as the induced subgraph comprising $\wt{V}_i^L=\{\text{indices of the}~p~\text{largest entries of}~[P^{L}]_{i,:}\}$.
When $p$ is sufficiently large, $\wt{G}_i\supseteq G_i$ for every $i$, ensuring that~\cref{eq:receptive-field} is satisfied.

\scm{Small variation of spectral embeddings on ego-graphs.} We start from a key observation that the variation of the spectral embeddings on ego-graphs $\wt{Z}_i=U_{G_i}\tran f_\theta(A_{G_i}, X_{G_i})$ across all the nodes $i\in[n]$ is small, when $p$ is not too small, for all possible $\theta$. This is formalized as follows.
\begin{assumption}[Bounded Variation of $\wt{Z}_i$]
For large enough graph size $n$ and ego-graph size $p$, we assume for any model parameter $\theta\in\Theta$, $\text{RSD}(\wt{Z}) \coloneqq \sqrt{\frac{1}{n}\sum_{i\in[n]} \|\wt{Z}_i-\wt{Z}\|_{F}}\big/\|\wt{Z}\|_{F} < B$, i.e., the relative standard deviation (RSD) of $\wt{Z}$ is upper-bounded by a constant $B>0$ independent of $\theta$, where $\wt{Z} = \frac{1}{n} \sum_{i\in[n]} \wt{Z}_i$ is the node-wise average of $\wt{Z}_i$.
\end{assumption}\label{assump:bounded-spectral-variance}
In~\cref{fig:spectral-concentration}, we plot $\text{RSD}(\wt{Z}_i)$ versus $p$ on the Cora dataset~\citep{yang2016revisiting}, where we find the RSD drops vastly when $p$ increases. 
The intuition behind this phenomenon is that many real-world graphs (e.g., citation and social networks) are often \emph{self-similar}, where the spectral representations of large-enough ego-graphs are close to the full graphs'.

\scm{Approximately decompose the ego-graph-wise graph coreset objective~\cref{eq:ego-spacial-coreset} in spectral domain.} We can re-write \cref{eq:receptive-field} in the spectral domain of each ego-graph as
\begin{equation}
\label{eq:spectral-receptive-filed}
    \ell([f_\theta(A_{G_i},X_{G_i})]_{1,:},~y_i)=\ell([U_{G_i}]_{1,:}\wt{Z}_i, y_i),
\tag{SRF}
\end{equation}
since $[f_\theta(A_{G_i},X_{G_i})]_{1,:}=[U_{G_i}U_{G_i}\tran f_\theta(A_{G_i},X_{G_i})]_{1,:}=[U_{G_i}]_{1,:}U_{G_i}\tran $\\$ f_\theta(A_{G_i},X_{G_i})=[U_{G_i}]_{1,:}\wt{Z}_i$.
We now denote $\cv_i\coloneqq[U_{G_i}]_{1,:}\tran\in\mathbb{R}^{p}$ (not an eigenvector), $\wt{\ell}_i(\wt{Z})=\ell(\cv_i\tran\wt{Z},~y_i)$, and $\wt{\mathcal{L}}(\wt{Z})=\frac{1}{n_t}\sum_{i\in[n_t]}\wt{\ell}_i(\wt{Z})$.
Since by~\cref{assump:bounded-spectral-variance}, we assume $\wt{Z}_i\approx\wt{Z}$ for all $\theta\in\Theta$ and $i\in[n]$, we propose to approximately achieve the goal of ego-graph-wise coreset (\cref{eq:ego-spacial-coreset}) by: 
(1) finding the subset of labeled nodes to approximate the average spectral embedding,
\begin{equation}
\label{eq:graph-average-coreset}    \min_{w\na\in\mathcal{W}}\max_{\theta\in\Theta}\big\|\sum_{i\in [n_{t}]}w\na_i\cdot \wt{Z}_i - \wt{Z}~\big\|_{F},
\tag{\textcolor{green}{NAC}}
\end{equation}
which we call \wcm{node-wise average coresets}; and
(2) finding the subset of labeled nodes to approximate the node-classification loss,
\begin{equation}
\label{eq:linear-classification-coreset}
    \min_{w\lc\in\mathcal{W}}\max_{\wt{Z}}\big|\sum_{i\in [n_{t}]}w\lc_i\cdot\wt\ell_i(\wt{Z})-\wt{\mathcal{L}}(\wt{Z})\big|,
\tag{\textcolor{green}{LCC}}
\end{equation}
where now the average spectral embedding $\wt{Z}$ is treated as an unknown parameter. Since $\wt{Z}$ is the output embedding and $\wt{\ell}_i(\wt{Z})=\ell(\cv_i\tran\wt{Z},~y_i)$ is a linear classification loss, \cref{eq:linear-classification-coreset} is the \wcm{linear classification coresets}.
Although the optimal sample weights $w\na$ and $w\lc$ (where the superscript $\na$ stands for average and $\lc$ stands for classification) are different, we further require \emph{the corresponding subsets of nodes coincide}, i.e., $V_{w\na}=V_{w\lc}$, and this is realized by the combined coreset algorithm (see~\cref{alg:sggc}).
Moreover, given~\cref{assump:bounded-spectral-variance}, if we can upper-bound the errors in~\cref{eq:graph-average-coreset,eq:linear-classification-coreset} through the combined coreset algorithm, we can show the approximation error on the node classification loss is also upper-bounded (see~\cref{thm:error-gnn-loss}).

The remaining of this section analyzes how to solve the two coreset selection problems one by one, while we defer the combined greedy algorithm and theoretical guarantees to~\cref{sec:algorithm}.

\subsection{Graph Node-wise Average Coresets}
\label{subsec:node-average-coreset}

\scm{Solving node-wise average coresets (\cref{eq:graph-average-coreset}) approximately without evaluating the spectral embeddings.} For the node-wise average coresets (\cref{eq:graph-average-coreset}), since the evaluation of a single spectral embedding $\wt{Z}_i$ is expensive, we ask: \emph{is it possible to find the coresets approximately without evaluating any $\wt{Z}_i$?}
Surprisingly, this is possible because the spectral embedding $\wt{Z}_i$ is a ``\emph{smooth}'' function of nodes $i$ on the graph.
Here, ``\emph{smothness}'' refers to the phenomena that $\wt{Z}_i$ (as a function of node $i$) varies little across edges, i.e., $\wt{Z}_i\approx \wt{Z}_j$ if $A_{i,j}=1$.
The intuition behind this is simple: ego-graphs of connected nodes have a large overlap $\wt{G}_i\cap\wt{G}_j$, and thus the resulted output embedding is similar no matter what parameter $\theta$ is used.

\scm{Spectral characterization of smoothness.} The spectral transformation can again be used to characterize the degree of smoothness since the eigenvalue $\lambda_i$ represents the smoothness of eigenvector $\egv_i$.
For an entry of the spectral embedding $[\wt{Z}_i]_{a,b}$, we can construct an $n$-dimensional vector $\wt{\mathbf{z}}^{(a,b)}=\big[[\wt{Z}_1]_{a,b},\ldots,[\wt{Z}_n]_{a,b}\big]$\\$\in\mathbb{R}^{n}$ by collecting the corresponding entry of the spectral embedding of each node.
Then, we want to show the inner product $\langle\wt{\mathbf{z}}^{(a,b)}, \egv_i\rangle$ is larger for smaller eigenvalue $\lambda_i$.
Actually, this can be done by first considering the spectral representation of the inputs, i.e., $\wt{X}_i=U_{G_i}\tran X_{G_i}$, where we can similarly define $\wt{\mathbf{x}}^{(a,b)}=\big[[\wt{X}_1]_{a,b},\ldots,[\wt{X}_n]_{a,b}\big]$ and show that if the node features are \textit{i.i.d.} unit Gaussians, then in expectation $\langle\wt{\mathbf{x}}^{(a,b)}, \egv_i\rangle\propto (1-\frac12\lambda_i)^L$; see \cref{lem:smooth-input} in~\cref{apd:proofs}.
Second, we note that the spectral behavior of message-passing GNNs $f_\theta(A,X)$ (\cref{eq:gnn}) is completely characterized by its convolution matrix~\citep{balcilar2021analyzing}; see \cref{fig:gnn-spectral} for practical observations.
Based on this, we can show the corresponding GNN function in the spectral domain $\wt{f}_\theta(\cdot)=U\tran f_\theta(A,U\cdot)$ is Lipschitz continuous if all of the linear weights $W^{(l)}$ in~\cref{eq:gnn} have bounded operator norms (see~\cref{lem:lipschitz-gnn} in~\cref{apd:proofs}).
Based on these results, we can formally characterize the smoothness of spectral embeddings (see~\cref{prop:smooth-embedding} in~\cref{apd:proofs}).

\scm{Upper-bound on the node-wise average error.} Following the work in~\citep{linderman2020numerical,vahidian2020coresets}, and based on~\cref{prop:smooth-embedding}, we can obtain an upper-bound on the node-wise average error $\|\sum_{i\in [n_{t}]}w\na_i\cdot \wt{Z}_i - \wt{Z}\|_{F}\leq M\cdot \|P\mathbf{w}\na-\frac{1}{n}\vones\|$ (see~\cref{thm:node-error-bound} in~\cref{apd:proofs}), where $\mathbf{w}\na=\sum_{i\in[n_t]}w\na_i \bm{\delta}_i\in\mathbb{R}^{n}$, $\bm{\delta}_i$ is the unit vector whose $i$-th entry is one, and $\vones$ is the vector of ones.
We then propose to optimize the upper-bound $\|P\mathbf{w}\na-\frac{1}{n}\vones\|$ which does not depend on $\wt{Z}_i$, enabling us to approximately solve the node-wise average coreset without evaluating a single $\wt{Z}_i$.
\citep{vahidian2020coresets} propose to optimize $ \|P\mathbf{w}\na-\frac{1}{n}\vones\|$ using a variant of the greedy geodesic iterative ascent (GIGA) algorithm~\citep{campbell2018bayesian}, and we follow their approach (see~\cref{sec:algorithm} for details).
\begin{theorem}[\textbf{Upper-bound on the Error Approximating Node-wise Average}]
\label{thm:node-error-bound}
Under all assumptions of~\cref{prop:smooth-embedding}, we have $\|\sum_{i\in [n_{t}]}w\na_i\cdot \wt{Z}_i - \wt{Z}\|_{F}\leq M\cdot \|P\mathbf{w}\na-\frac{1}{n}\vones\|$ for some constant $M>0$.
\end{theorem}

\begin{proof}
The proof mostly follows from the proof of Theorem 1 in~\citep{linderman2020numerical}. We begin with decomposing an arbitrary $(a,b)$-th entry of $\sum_{i\in [n_{t}]}w\na_i\cdot \wt{Z}_i-\wt{Z}$ in the spectral domain. 
\begin{equation}
\notag
\begin{aligned}
&\textstyle \sum_{i\in [n_{t}]}w\na_i\cdot \wt{\mathbf{z}}^{(a,b)}_i-\frac{1}{n}\sum_{j\in[n]}\wt{\mathbf{z}}^{(a,b)}_j
\\=&\textstyle \sum_{j\in[n]}(\sum_{i\in[n_t]}w\na_i\bm{\delta}_i-\frac{1}{n})\wt{\mathbf{z}}^{(a,b)}_j
\\=&\textstyle \sum_{k\in[n]}\big\langle\sum_{i\in[n_t]}w\na_i\bm{\delta}_i-\frac{1}{n}, \egv_k\big\rangle \langle\egv_k, {\mathbf{z}}^{(a,b)}\rangle 
\\\leq&\textstyle M'\cdot \sum_{k\in[n]}\big\langle\sum_{i\in[n_t]}w\na_i\bm{\delta}_i-\frac{1}{n}, \egv_k\big\rangle (1-\frac12\lambda_k)^L.
\end{aligned}
\end{equation}

Thus we have,
\begin{equation}
\notag
\begin{aligned}
&\textstyle \big|\sum_{i\in [n_{t}]}w\na_i\cdot \wt{\mathbf{z}}^{(a,b)}_i-\frac{1}{n}\sum_{j\in[n]}\wt{\mathbf{z}}^{(a,b)}_j\big|^2 \\\leq&\textstyle M'^2\cdot \sum_{k\in[n]}\big|\big\langle\sum_{i\in[n_t]}w\na_i\bm{\delta}_i-\frac{1}{n}, \egv_k\big\rangle\big|^2 (1-\frac12\lambda_k)^{2L}
\\=&\textstyle M'^2\big\| \sum_{k\in[n]} (1-\frac12\lambda_k)^{L}\big\langle\sum_{i\in[n_t]}w\na_i\bm{\delta}_i-\frac{1}{n}, \egv_k\big\rangle \egv_k \big\|^2\\=&\textstyle M'^2\cdot \big\|P^L (\sum_{i\in[n_t]}w\na_i\bm{\delta}_i-\frac{1}{n})\big\|\\=&\textstyle M'^2 \cdot \|P\mathbf{w}\na-\frac{1}{n}\vones\|^2.
\end{aligned}
\end{equation}

Thus we conclude,
\begin{equation}
\notag
\|\sum_{i\in [n_{t}]}w\na_i\cdot \wt{Z}_i - \wt{Z}\|_{F}\leq M\cdot \|P\mathbf{w}\na-\frac{1}{n}\vones\|,
\end{equation}
where $M=M'\cdot p d$. Here $p$ is the (diffusion) ego-graph size and $d$ is the number of features per node.
\end{proof}

\subsection{Spectral Linear Classification Coresets}
\label{subsec:spectral-classification-coreset}

There are more available approaches to solve the linear classification coreset~\cref{eq:linear-classification-coreset}, and we adopt the submodular maximization formulation in~\citep{mirzasoleiman2020coresets}.

\scm{Submodular maximization formulation of linear classification coreset.} Following~\citep{mirzasoleiman2020coresets}, we can show the approximation error in~\cref{eq:linear-classification-coreset} can be upper-bounded by a set function $H(\cdot)$, i.e., $|\sum_{i\in [n_{t}]}w\lc_i\cdot\wt\ell_i(\wt{Z})-\wt{\mathcal{L}}(\wt{Z})|\leq H(V_{w\lc})$, where $H(V_{w\lc})\coloneqq\sum_{i\in[n_t]}\min_{j\in V_{w\lc}} \max_{\wt{Z}}|\ell_i(\wt{Z})-\ell_j(\wt{Z})|$ (see~\cref{lem:submodular-error-bound} in~\cref{apd:proofs}). 
Then, by introducing an auxiliary node $\{i_0\}$, we can define a submodular function $F(V)\coloneqq H(\{i_0\})-H(V\cup\{i_0\})$, and formulate the coreset selection as a submodular set-cover problem. 
Due to the efficiency constraints, \citep{mirzasoleiman2020coresets} propose to solve the submodular maximization problem instead, $\max_{w\lc\in\mathcal{W}}F(V_{w\lc})$, which is dual to the original submodular cover formulation. We follow this approach and adopt their CRAIG (CoResets for Accelerating Incremental Gradient descent) algorithm for the linear classification coreset.
It is worth noting that, although the CRAIG formulation discussed above can be used to solve the original ego-graph-wise coreset problem (\cref{eq:ego-spacial-coreset}) directly, it suffers from a much larger complexity as we have to forward- and backward-pass through the GNN all the time, and evaluate all ego-graph embeddings explicitly.
\section{Algorithm and Theoretical Analysis}
\label{sec:algorithm}

\scm{The spectral greedy graph coresets (SGGC) algorithm.} We now describe how we combine the two greedy algorithms, GIGA and CRIAG, to achieve both objectives respectively, with an extra constraint that they find the same subset of nodes, i.e., $V_{w\na}=V_{w\lc}$.
The idea is to incorporate the submodular cost $F(V_{w\lc})=F(V_{w\na})$ into the SCGIGA's objective.
Through the introduction of a hyperparameter $0<\kappa<1$, we change the objective of the node-wise average coreset to be $ \|P\mathbf{w}\na-\frac{1}{n}\vones\|-\kappa F(V_{w\na})$.
Now, the submodular cost $F(V_{w\na})$ can be understood as a selection cost, and the new objective can be solved by a relaxation on the GIGA algorithm, which is called SCGIGA as discussed in~\citep{vahidian2020coresets}.
The complete pseudo-code is shown below (see~\cref{apd:proofs} for more details).

\begin{algorithm}
\caption{\label{alg:sggc}Spectral greedy graph coresets (SGGC).}
\KwIn{Diffusion matrix $P=\frac12I_n+\frac12D^{-1}A$, coreset size $c$, hyperparameter $0<\kappa<1$.}
Initialize weights $w_0\na\leftarrow \mathbf{0}$, $w_0\lc\leftarrow\mathbf{0}$\;
\For{$t=0,\ldots,c-1$}{
    Compute $P(w\na_t)=\sum_{i\in[n_t]}[w\na_t]_i\frac{P_{:,i}}{\|P_{:,i}\|}$\;
    Compute $\mathbf{a}_t \leftarrow \frac{\vones-\left\langle\vones, P\left(w_t\na\right)\right\rangle P\left(w_t\na\right)}{\left\|\vones-\left\langle\vones, P\left(w_t\na\right)\right\rangle P\left(w_t\na\right)\right\|}$, and for each $i\in[n_t]$, $\mathbf{b}_t^i \leftarrow \frac{P_i-\left\langle P_i, P\left(w_t\na\right)\right\rangle P\left(w_t\na\right)}{\left\|P_i-\left\langle P_i, P\left(w_t\na\right)\right\rangle P\left(w_t\na\right)\right\|}$\;
    Find subset $V_t=\{i\in[n_t]\mid \langle\mathbf{a}_t,~\mathbf{b}_t^i\rangle \geq\kappa\cdot\max_{j\in[n_t]}\langle\mathbf{a}_t,~\mathbf{b}_t^j\rangle\}$\;
    Select node $i^*=\arg\max_{i\in V_t}F(\{i\}\cup V_{w\na})-F(V_{w\na})$\;
    Compute $\zeta_0=\langle\frac{\vones}{\sqrt{n}}, P_{i^*}\rangle, \zeta_1=\langle\frac{\vones}{\sqrt{n}}, P\left(w_t)\right\rangle, \zeta_2=\langle P_{i^*}, P\left(w_t\right)\rangle$, and $\eta_t \leftarrow \frac{\zeta_0-\zeta_1 \zeta_2}{\left(\zeta_0-\zeta_1 \zeta_2\right)+\left(\zeta_1-\zeta_0 \zeta_2\right)}$\;
    Update weights $w_{t+1}\na \leftarrow \frac{\left(1-\eta_t\right) w_t\na+\eta_t \bm{\delta}_{i^*}}{\left\|\left(1-\eta_t\right) P\left(w_t\na\right)+\eta_t P_{i^*}\right\|}$\;
}
Compute $[w\na]_i\leftarrow\frac{1}{n\|P_{:,i}\|\|\sum_{j\in[n_t]}[w\na_c]_j P_{:,j}\|}[w_c\na]_i\quad\forall i\in[n_t]$\;
Compute $w\lc=\sum_{j\in[n_t]} \vones\big\{i=\arg \min _{k\in V_{w\na}} \max_{\wt{Z}} |\wt{\ell}_j(\wt{Z})-\wt{\ell}_k(\wt{Z})|\big\}$\;
Combine $w_i\leftarrow w\na_i\cdot w\lc_i$ for each $i\in[n_t]$, and normalize $w\leftarrow w/\|w\|_1$\;
\Return{coreset $V_w$, weights $w$}
\end{algorithm}

\scm{Theoretical guarantees of SGGC.} Based on the correctness theorems of SCGIGA and CRAIG, and~\cref{assump:bounded-spectral-variance}, we can prove the following error-bound on the node-classification loss, which shows SGGC approximately solves the graph coresets problem (\cref{eq:graph-coreset}) (see~\cref{apd:proofs}).
\begin{theorem}[Error-Bound on Node Classification Loss]
\label{thm:error-gnn-loss}
If both~\cref{eq:node-coreset} and~\cref{eq:linear-classification-coreset} have bounded errors and~\cref{assump:bounded-spectral-variance} holds, then we have, $\max_{\theta\in\Theta}\big|\sum_{i\in [n_{t}]}w\na_i w\lc_i\cdot$\\$\ell\big([f_\theta(A_{G_i},X_{G_i})]_{1,:},~y_i\big)-\mathcal{L}(\theta)\big| < \epsilon$, where $\epsilon$ does not depend on the coreset size $c$ and the number of training nodes $n_t$.
\end{theorem}
\begin{proof}{thm:error-gnn-loss}
First, using~\cref{alg:sggc} and by the Theorem 2 of~\cite{vahidian2020coresets}, we have $\|P\mathbf{w}\na-\frac{1}{n}\vones\|\leq \sqrt{\frac{1}{n}-(\frac{\kappa}{n})^2}\cdot O((1-\kappa^2\epsilon^2)^{c/2})$ for some $\epsilon>0$ and where $c=|V_w|$ is the size of selected coreset. By~\cref{thm:node-error-bound}, we have the upper-bound on the objective of the node-wise average coreset (\cref{eq:node-coreset}), $\|\sum_{i\in [n_{t}]}w\na_i\cdot \wt{Z}_i - \wt{Z}\|_{F}<M\na$ for some $M\na>0$. Then, by Theorem 1 of~\citep{vahidian2020coresets}, we know $F(V_w)=H(\{i_0\})-H(V_w\cup\{i_0\})\geq O(\frac{1-\kappa}{\sqrt{n}})$. Thus by~\cref{lem:submodular-error-bound}, we have an upper-bound on, $\max_{\wt{Z}}|\ell_i(\wt{Z})-\ell_j(\wt{Z})|\leq M\lc$ for some $M\lc>0$ for any $i\in V_w$ and $j\in[n]$. Note the $w\na$ and $w\lc$ above is the output of~\cref{alg:sggc} and we have $V_{w\na}=V_{w\lc}=V_w$.

Second, by repeatedly using the second bound above, we can get $|\sum_{i\in [n_{t}]}w\na_i w\lc_i\cdot\wt\ell_i(\wt{Z}_i)-\sum_{j\in[n_t]}\sum_{k\in[n]}w\na_j\wt\ell_k(\wt{Z}_j)|\leq M\lc$, where $\zeta: [n]\to V_w$ is the mapping described in~\cref{lem:submodular-error-bound}. While, by the first bound above, \cref{assump:bounded-spectral-variance}, the Jensen-bound and the Lipschitzness of $\wt\ell_k$ (the Lipschitz coefficient is an absolute constant), we can derive, $|\sum_{i\in[n_t]}w\na_i \wt\ell_k(\wt{Z}_i)-\wt\ell_k(\wt{Z}_k)|\leq |\sum_{i\in[n_t]}w\na_i \wt\ell_k(\wt{Z}_i)-\wt\ell_k(\wt{Z})|+|\wt\ell_k(\wt{Z})-\wt\ell_k(\wt{Z}_k)|\leq O(M\na+B\|\wt{Z}\|)$. Combining the two inequalities, we conclude the proof.
\end{proof}
\section{Related Work}
\label{sec:related}

\begin{table*}[htbp!]
\caption{SGGC is better than other model-agnostic/based coresets, graph coarsening, and comparable to graph condensation. We train $2$-layer GCNs on the coreset/coarsened/condensed graphs and report the test accuracy. OOT and OOM refer to out-of-time/memory.}\label{tab:compare}
\adjustbox{max width=\textwidth}{%
{\renewcommand{\arraystretch}{1.3}%
{\Large
\begin{tabular}{ccccccccccccccc}
    \hline
    \multirow{2}{*}{\textbf{Dataset}} & \multirow{2}{*}{Ratio} & \multicolumn{3}{c}{Model-Agnostic Coresets} & \multicolumn{6}{c}{Model-Based Coresets}                                                                  & Graph Reduction & Ours                  & Data Condense & Oracle                        \\
                                      &                        & Uniform                & Herding      & K-Center              & Forgetting   & Cal                   & CRAIG                 & Glister               & GraNd                         & GradMatch             & Corasening      & SGGC                  & GCond         & Full Graph                    \\ \hline
    \multirow{3}{*}{Cora}             & 15\%                   & 67.7$\pm$4.5           & 66.1$\pm$1.2 & 64.3$\pm$4.8          & 65.4$\pm$3.1 & 71.6$\pm$1.0          & 68.4$\pm$4.4          & 65.6$\pm$5.6          & \textbf{71.9$\pm$1.7}         & 72.0$\pm$1.3          & ---             & \textbf{72.9$\pm$0.6} & ---           & \multirow{3}{*}{81.2$\pm$0.4} \\
                                      & 25\%                   & 71.8$\pm$4.2           & 69.9$\pm$1.0 & 72.6$\pm$2.5          & 72.6$\pm$3.5 & 75.3$\pm$1.5          & 74.4$\pm$1.7          & 74.3$\pm$2.4          & 74.4$\pm$1.5                  & 74.7$\pm$2.3          & 31.2$\pm$0.2    & \textbf{78.6$\pm$1.0} & 79.8$\pm$1.3  &                               \\
                                      & 50\%                   & 78.3$\pm$2.2           & 70.8$\pm$0.4 & 78.9$\pm$1.0          & 76.1$\pm$1.1 & \textbf{80.7$\pm$0.5} & 78.2$\pm$2.0          & 78.3$\pm$2.0          & \textbf{79.3$\pm$0.8}         & \textbf{80.2$\pm$0.5} & 65.2$\pm$0.6    & \textbf{80.2$\pm$0.8} & 80.1$\pm$0.6  &                               \\ \hline
    \multirow{3}{*}{Citeseer}         & 15\%                   & 53.6$\pm$7.9           & 46.1$\pm$1.6 & 47.5$\pm$6.3          & 51.5$\pm$4.9 & 53.2$\pm$2.0          & 55.4$\pm$6.7          & 54.0$\pm$5.0          & 57.0$\pm$3.9                  & 58.8$\pm$3.9          & ---             & \textbf{63.7$\pm$3.1} & ---           & \multirow{3}{*}{70.6$\pm$0.9} \\
                                      & 25\%                   & 61.7$\pm$3.2           & 54.9$\pm$3.9 & 61.6$\pm$4.0          & 55.3$\pm$5.5 & 56.1$\pm$2.8          & 59.5$\pm$4.3          & \textbf{62.0$\pm$5.5} & 64.4$\pm$1.5                  & \textbf{66.0$\pm$1.5} & 52.2$\pm$0.4    & \textbf{67.2$\pm$2.4} & 70.5$\pm$1.2  &                               \\
                                      & 50\%                   & 66.9$\pm$1.7           & 68.7$\pm$0.5 & 65.6$\pm$1.6          & 67.6$\pm$0.8 & 68.2$\pm$0.8          & 67.9$\pm$2.2          & 67.9$\pm$1.4          & 70.5$\pm$0.8                  & \textbf{70.7$\pm$0.5} & 59.0$\pm$0.5    & 68.4$\pm$0.9          & 70.6$\pm$0.9  &                               \\ \hline
    \multirow{3}{*}{Pubmed}           & 15\%                   & 65.7$\pm$4.5           & 61.9$\pm$1.0 & \textbf{69.0$\pm$4.8} & 65.4$\pm$4.9 & 71.7$\pm$0.8          & \textbf{73.2$\pm$4.1} & 65.6$\pm$5.5          & 62.0$\pm$1.0                  & 65.3$\pm$4.5          & ---             & \textbf{72.5$\pm$1.5} & ---           & \multirow{3}{*}{79.3$\pm$0.6} \\
                                      & 25\%                   & 71.1$\pm$1.8           & 65.9$\pm$0.4 & \textbf{73.3$\pm$2.6} & 69.0$\pm$2.5 & \textbf{74.7$\pm$1.7} & 71.0$\pm$3.3          & 71.5$\pm$3.2          & 70.6$\pm$2.3                  & 71.1$\pm$1.4          & ---             & \textbf{75.8$\pm$1.6} & ---           &                               \\
                                      & 50\%                   & 75.3$\pm$1.1           & 72.2$\pm$0.6 & \textbf{77.8$\pm$1.3} & 72.5$\pm$2.1 & \textbf{77.3$\pm$1.1} & 74.4$\pm$1.7          & 75.1$\pm$2.0          & 76.0$\pm$1.2                  & 74.8$\pm$1.1          & ---             & \textbf{76.5$\pm$0.6} & ---           &                               \\ \hline
    \multirow{3}{*}{Flickr}           & 0.2\%                  & 47.1$\pm$1.6           & 45.5$\pm$1.3 & 46.9$\pm$0.9          & 46.0$\pm$1.2 & OOT                   & OOT                   & \textbf{48.0$\pm$1.0} & \textbf{48.1$\pm$0.4}         & \textbf{48.2$\pm$0.2} & 41.9$\pm$0.2    & \textbf{48.4$\pm$0.8} & 46.5$\pm$0.4  & \multirow{3}{*}{49.1$\pm$0.7} \\
                                      & 1.0\%                    & 48.4$\pm$1.1           & 46.7$\pm$0.3 & 47.5$\pm$0.9          & 47.7$\pm$2.4 & OOT                   & OOT                   & \textbf{48.5$\pm$0.7} & \textbf{48.8$\pm$0.5}         & \textbf{48.7$\pm$0.6} & 44.5$\pm$0.1    & \textbf{49.0$\pm$0.6} & 47.1$\pm$0.1  &                               \\ 
                                      & 2.0\%                    & 47.0$\pm$1.1           & 45.5$\pm$0.6 & 46.9$\pm$0.7          & 46.6$\pm$1.6 & OOT                   & OOT                   & 47.4$\pm$0.9          & OOM                           & 48.5$\pm$0.6          & ---             & \textbf{64.4$\pm$0.4} & ---           &                                \\ \hline
    \multirow{3}{*}{ogbn-arxiv}            & 0.5\%                  & 58.4$\pm$1.5           & 45.7$\pm$4.4 & 56.8$\pm$2.8          & 55.5$\pm$2.4 & OOT                   & OOT                   & 57.2$\pm$2.1          & OOM                           & 53.4$\pm$1.9          & 43.5$\pm$0.2    & \textbf{59.7$\pm$1.5} & 63.2$\pm$0.3  &  \multirow{3}{*}{70.9$\pm$0.2}                  \\
                                      & 1.0\%                    & 62.0$\pm$0.9           & 47.6$\pm$0.4 & 60.7$\pm$0.8          & 60.4$\pm$1.9 & OOT                   & OOT                   & 62.2$\pm$1.3          & OOM                           & 56.5$\pm$1.7          & 50.4$\pm$0.1     & \textbf{62.5$\pm$0.9} & 64.0$\pm$0.4 &                                \\ 
                                      & 2.0\%                    & 64.7$\pm$0.5           & 56.5$\pm$0.5 & 62.4$\pm$0.9          & 62.8$\pm$2.4 & OOT                   & OOT                   & 64.2$\pm$1.1          & OOM                           & 58.6$\pm$1.0          & ---             & \textbf{64.4$\pm$0.4} & ---           &                                \\ \hline                                                                                                                                                                                                                                                                                                                                  
    \multirow{2}{*}{ogbn-products}         & 0.05\%                 & \textbf{46.8$\pm$1.2}  & 31.9$\pm$0.5 & 35.9$\pm$1.9          & 32.9$\pm$4.8 & OOT                   & OOT                   & OOM                   & OOM                           & OOM                   & ---             & \textbf{46.3$\pm$4.1} & ---           &  \multirow{2}{*}{75.6$\pm$0.2}                  \\
                                      & 0.15\%                 & \textbf{53.0$\pm$1.0}  & 36.5$\pm$0.3 & 47.6$\pm$0.8          & 42.0$\pm$3.7 & OOT                   & OOT                   & OOM                   & OOM                           & OOM                   & ---             & \textbf{53.6$\pm$1.2} & ---           &                                \\ \hline
    \multirow{2}{*}{Reddit}           & 0.1\%                  & 27.4$\pm$4.6           & 18.5$\pm$3.5 & 22.5$\pm$4.5          & 26.4$\pm$1.0 & OOT                   & OOT                   & OOM                   & OOM                           & 19.4$\pm$3.5          & ---             & \textbf{38.4$\pm$3.4} & ---           & \multirow{2}{*}{92.2$\pm$0.6}                 \\
                                      & 0.2\%                  & 40.7$\pm$7.2           & 17.0$\pm$4.0 & 20.0$\pm$3.1          & 39.7$\pm$3.5 & OOT                   & OOT                   & OOM                   & OOM                           & 18.3$\pm$3.0          & ---             & \textbf{48.6$\pm$4.6}          & ---           &                                \\ \hline                                         
    
    \end{tabular}
}}}
\end{table*}

In this section, we review general coreset methods, graph coresets, and other graph reduction methods, as well as graph condensation that adapts dataset condensation to graph (see~\cref{apd:related}).

Early coreset selection methods consider unsupervised learning problems, e.g., clustering.
\wcm{Coreset selection} methods choose samples that are important for training based on certain heuristic criteria.
They are usually \wcm{model-agnostic}; for example, \textit{Herding} coreset~\citep{welling2009herding} selects the closest samples to the cluster centers.
\textit{K-center} coreset~\citep{farahani2009facility} picks multiple center points such that the largest distance between a data point and its nearest center is minimized.
In recent years, more coreset methods consider the supervised learning setup and propose many \wcm{model-based} heuristic criteria, such as maximizing the diversity of selected samples in the gradient space~\citep{aljundi2019gradient}, discovering cluster centers of model embedding~\citep{sener2018active}, and choosing samples with the largest negative implicit gradient~\citep{borsos2020coresets}.

\wcm{Graph coreset selection} is a non-trivial generalization of the above-mentioned coreset methods given the interdependent nature of graph nodes. The very few off-the-shelf graph coreset algorithms are designed for graph clustering~\citep{baker2020coresets,braverman2021coresets} and are not optimal for the training of GNNs.

\wcm{Graph sparsification}~\citep{batson2013spectral,satuluri2011local} and \wcm{graph coarsening}~\citep{loukas2018spectrally,loukas2019graph,huang2021scaling,cai2020graph} algorithms are usually designed to preserve specific graph properties like graph spectrum and graph clustering. Such objectives often need to be aligned with the optimization of downstream GNNs and are shown to be sub-optimal in preserving the information to train GNNs well~\citep{jin2021graph}.

\wcm{Graph condensation}~\citep{jin2021graph} adopts the recent \emph{dataset condensation} approach which \textit{synthesizes} informative samples rather than selecting from given ones. 
Although graph condensation achieves the state-of-the-art performance for preserving GNNs' performance on the simplified graph, it suffers from two severe issues: \textbf{(1)} extremely long condensation training time; and \textbf{(2)} poor generalizability across GNN architectures.
Subsequent work aims to apply a more efficient distribution-matching algorithm~\citep{zhao2021datasetb,wang2022cafe} of dataset condensation to graph~\citep{liu2022graph} or speed up gradient-matching graph condensation by reducing the number of gradient-matching-steps~\citep{jin2022condensing}. While the efficiency issue of graph condensation is mitigated, the performance degradation on medium- and large-sized graphs~\citep{jin2021graph} still renders graph condensation practically meaningless.

\section{Experiments}
\label{sec:experiments}

In this section, we demonstrate the effectiveness and advantages of SGGC, together with some important proof-of-concept experiments and ablation studies that verify our design.
We also show the efficiency, architecture-generalizability, and robustness of SGGC.
We define the coreset ratio as $c/n_t$\footnote{Some paper like~\cite{jin2021graph} defines this ratio as $c/n$, which could be small even if we keep all training/labeled nodes, i.e., $c=n_t$ (e.g., on Cora and Citeseer) and is often misleading.}, where $c$ is the size of coreset, and $n_t$ is the number of training nodes in the original graph.
We train 2-layer GNNs with 256 hidden units and repeat every experiment 10 times.
See~\cref{apd:implement} for implementation details and~\cref{apd:experiments} for more results and ablation studies on more datasets.

\begin{table*}[t]
\caption{Selecting diffusion ego-graphs largely outperforms node-wise selection and achieves comparable performance to selecting standard ego-graphs with much smaller ego-graph sizes.}
    \label{tab:selecting}
    \centering
    \adjustbox{max width=\textwidth}{%
    {\renewcommand{\arraystretch}{1.20}%
    {\huge
    \begin{tabular}{cccccccccc} \toprule        \multirow{2}{*}{\begin{tabular}[c]{@{}c@{}}Dataset\\[-5pt] Ratio\end{tabular}}  & \multirow{2}{*}{\begin{tabular}[c]{@{}c@{}}Selection\\[-5pt] Strategy\end{tabular}} & \multicolumn{2}{c}{Model-Agnostic} & \multicolumn{2}{c}{Model-Based} & \multicolumn{2}{c}{Ablation Baselines} & Ours & Oracle                        \\
                              &                                                                               & Uniform               & K-Center                & CRAIG                 & Glister               & CRAIG-Linear            & SCGIGA                & \textbf{SGGC}                  & Full Graph                    \\ \hline
    \multirow{3}{*}{\begin{tabular}[c]{@{}c@{}}Cora\\ 25\%\end{tabular}}     & Node                                                                          & 63.3$\pm$2.7          & 67.7$\pm$2.7            & 64.6$\pm$4.2          & 61.9$\pm$5.5          & 64.1$\pm$4.0           & 63.0$\pm$2.0          & 70.3$\pm$1.2          & \multirow{3}{*}{81.2$\pm$0.4} \\
                              & Std. Ego                                                                      & \textbf{74.3$\pm$2.4} & \textbf{72.7$\pm$3.9}   & \textbf{74.5$\pm$3.3} & \textbf{73.7$\pm$1.9} & \textbf{73.0$\pm$3.4}  & \textbf{75.9$\pm$1.5} & \textbf{77.5$\pm$0.9} &                               \\
                              & Diff. Ego                                                                     & \textbf{73.7$\pm$1.1} & \textbf{72.6$\pm$2.2}   & \textbf{72.0$\pm$3.3} & \textbf{74.0$\pm$2.3} & \textbf{72.7$\pm$3.1}  & \textbf{76.7$\pm$1.9} & \textbf{76.8$\pm$1.0} &                               \\ \midrule
    \multirow{3}{*}{\begin{tabular}[c]{@{}c@{}}Citeseer\\ 25\%\end{tabular}} & Node                                                                          & \textbf{58.1$\pm$3.0} & 52.0$\pm$3.3            & 58.2$\pm$3.9          & 55.1$\pm$3.0          & 57.2$\pm$3.0           & \textbf{55.1$\pm$2.8} & 60.8$\pm$1.7          & \multirow{3}{*}{70.6$\pm$0.9} \\
                              & Std. Ego                                                                      & \textbf{61.8$\pm$4.8} & 56.8$\pm$4.4            & \textbf{60.0$\pm$5.6} & \textbf{59.7$\pm$5.9} & \textbf{61.7$\pm$5.8}  & \textbf{53.4$\pm$1.8} & \textbf{67.1$\pm$1.5} &                               \\
                              & Diff. Ego                                                                     & \textbf{61.7$\pm$3.2} & \textbf{61.6$\pm$4.0}   & \textbf{59.5$\pm$4.3} & \textbf{61.9$\pm$5.5} & \textbf{59.5$\pm$3.8}  & \textbf{54.6$\pm$2.7} & \textbf{67.2$\pm$2.4} &                               \\ \bottomrule
\end{tabular}
    }}}
\end{table*}

\begin{table}[htbp!]
\caption{The complete SGGC algorithm is better than the node-wise average coreset (SCGIGA) and the linear classification coreset (CRAIG-Linear) individually.} \label{tab:ablation}
    \centering
    \adjustbox{max width=0.75\textwidth}{%
    {\renewcommand{\arraystretch}{1.3}%
    {\huge
    \begin{tabular}{ccccccc} \toprule
    Dataset              & \multicolumn{2}{c}{Cora}         & \multicolumn{2}{c}{Citeseer}     & \multicolumn{2}{c}{Pubmed}       \\ \cmidrule(r){2-3} \cmidrule(lr){4-5} \cmidrule(l){6-7}
    Ratio                & 25\%            & 50\%           & 25\%            & 50\%           & 25\%            & 50\%           \\ \midrule
    Full Graph           & \multicolumn{2}{c}{81.2$\pm$0.4} & \multicolumn{2}{c}{70.6$\pm$0.9} & \multicolumn{2}{c}{79.3$\pm$0.6} \\ \arrayrulecolor{black!50} \midrule
    CRAIG-Linear          & 72.7$\pm$2.9                     & 78.1$\pm$1.1                     & 59.5$\pm$3.8            & \textbf{66.6$\pm$1.9}            & 71.6$\pm$3.8            & \textbf{75.4$\pm$2.3}   \\
    SCGIGA               & 76.7$\pm$1.4                     & 78.3$\pm$1.0                     & 54.6$\pm$2.8            & 66.9$\pm$1.1                     & 69.8$\pm$0.7            & 74.1$\pm$0.4     \\ \arrayrulecolor{black!50} \midrule
    \textbf{SGGC (Ours)} & \textbf{78.6$\pm$1.0}            & \textbf{80.2$\pm$1.1}            & \textbf{67.2$\pm$2.4}   & \textbf{68.4$\pm$0.9}            & \textbf{75.8$\pm$1.6}   & \textbf{76.5$\pm$0.6}   \\ \bottomrule
    \end{tabular}
    }}}
\end{table}

\scm{SGGC is better than other model-agnostic or model-based coresets and graph coarsening.}
Now, we demonstrate the effectiveness of SGGC in terms of the test performance (evaluated on the original graph) of GNNs trained on the coreset graph on seven node classification benchmarks with multiple coreset ratios $c/n_t$.
\cref{tab:compare}~presents the full results, where SGGC consistently achieves better performance than the other coreset methods and the graph coarsening approach.
Although graph condensation treats the condensed adjacency $A_w$ and node features $X_w$ as free learnable parameters (have less constraint than coreset methods), the performance is comparable to or even lower than SGGC on Cora and Flickr.
The advantages of SGGC are often more significant when the coreset ratio is small (e.g., on Citeseer with a 15\% ratio), indicating that SGGC is capable of extrapolating on the graph and finding informative ego-graphs when the budget is very limited.

Apart from the three small graphs (Cora, Citeseer, and Pubmed), we also consider two mid-scaled graphs (Flickr and ogbn-arxiv), a large-scale graph (ogbn-products) with more than two million nodes, and a much denser graph (Reddit) whose average node degree is around 50.
In~\cref{tab:compare}, we see that when scaling to larger and denser graphs, many model-based coreset methods, graph coarsening, and graph condensation are facing severe efficiency issues.
SGGC can run on ogbn-product with a coreset ratio $c/n_t=0.05\%$ within 43 minutes, while all the other model-based coresets (except Forgetting), graph coarsening, and graph condensation run out of time/memory.

\begin{figure*}[htbp!]
    \centering
    \begin{minipage}[t]{0.49\textwidth}
        \centering
        \includegraphics[width=.75\linewidth,trim={5pt 5pt 5pt 5pt},clip]{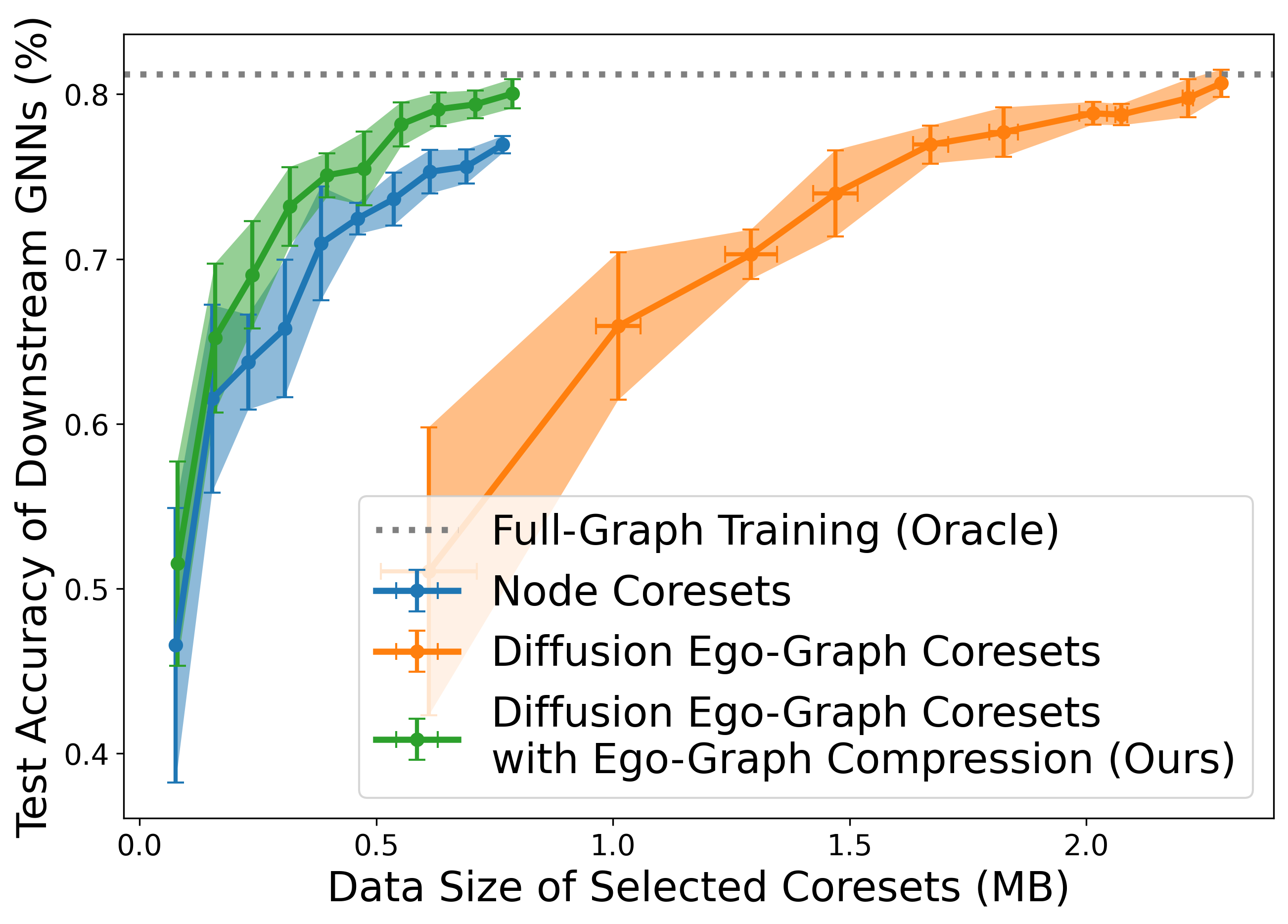}
        \caption{Test accuracy versus the selected data size of selecting nodes and diffusion ego-graphs with/without PCA-based compression of node attributes.}
        \label{fig:ego-compression}
    \end{minipage}
    \hfill
    \begin{minipage}[t]{0.49\textwidth}
        \centering
        \includegraphics[width=.80\linewidth,trim={5pt 0pt 5pt 5pt},clip]{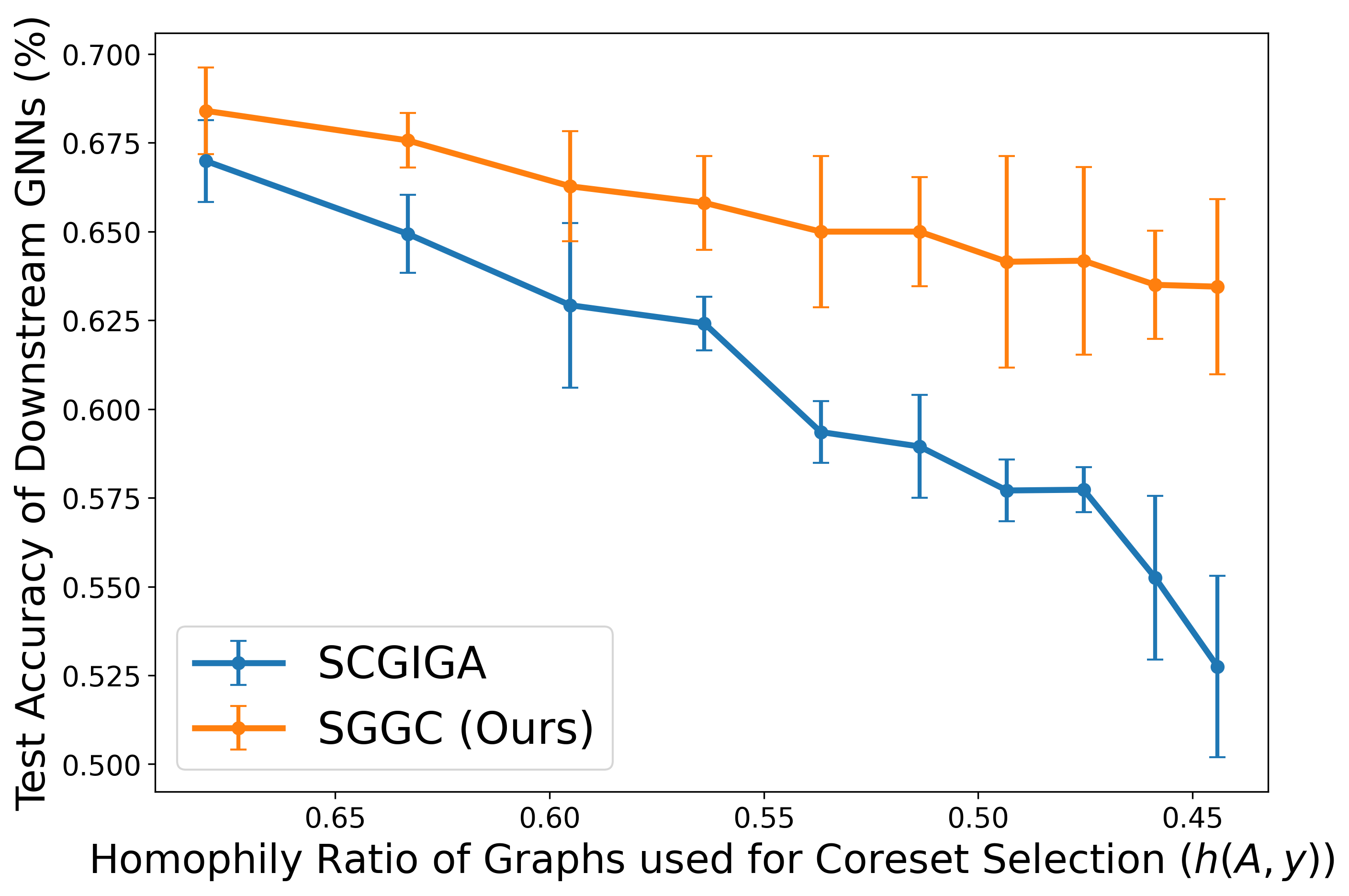}
        \caption{SGGC is more robust than SCGIGA on low-homophily graphs. We select the coresets on the edge-added graph with lower homophily, but train and test GCNs on the original graph.}
        \label{fig:ablation-homo}
    \end{minipage}
\end{figure*}

\scm{Selecting diffusion ego-graphs is better than selecting nodes and standard ego-graphs.}
We also verify that selecting diffusion ego-graphs is advantageous to selecting nodes.
In~\cref{fig:ego-compression}, we show that we can compress the diffusion ego-graphs to achieve data size comparable with node-wise selection without noticeably lowering the performance.
Ego-graph compression is based on the principal component analysis (PCA), and we compress the node features more when they are far away from the selected center nodes (see~\cref{apd:proofs}).
In~\cref{tab:selecting}, we compare the performance of various coreset methods with the three selection strategies, including selecting the standard ego-graphs or diffusion ego-graphs.
Not surprisingly, we see ego-graph selection strategies largely outperform node-wise selection (the largest gap is around 8\%).
Although selecting standard and diffusion ego-graphs often lead to similar performance, we note that, by selecting diffusion ego-graphs, we can achieve comparable performance with ego-graph-size $p=$8 or 16 on most datasets, which is much smaller than the average size of standard ego-graphs for $L=2$, e.g., around $36$ on Cora.

\scm{Ablation study: the two-stage SGGC algorithm is better than each stage individually.} It is important to verify that the combined coreset objective is better than the node-wise average coreset (\cref{eq:node-coreset}, implemented as SCGIGA~\citep{vahidian2020coresets} with zero selection cost) and the linear classification coreset (\cref{eq:linear-classification-coreset}, implemented as the CRAIG algorithm~\citep{mirzasoleiman2020coresets} with a linear model, denoted by CRAIG-Linear) individually (see~\cref{apd:implement} for details).
In~\cref{tab:ablation}, we see SGGC is consistently better than 
\textbf{(1)} CRAIG-Linear (outperformed by 3.8\% on average), which over-simplifies GNNs to a linear classifier and completely ignores the graph adjacency and 
\textbf{(2)} SCGIGA (outperformed by 4.4\% on average), which relies on a possibly wrong assumption that the node-wise classification loss is a ``smooth'' function of nodes over the graph.
Moreover, we find SGGC is more robust than SCGIGA, against the variations of \emph{homophily} in the underlying graph (as shown in~\cref{fig:ablation-homo}), where the homophily is defined as $h(A,\mathbf{y})=\frac{1}{|E|} \sum_{(i,j)\in E} \vones\{y_j=y_k\}$~\citep{ma2021homophily} (i.e., how likely the two end nodes of an edge are in the same class).
SCGIGA's performance is greatly degraded on low-homophily graphs because it assumes the node-wise classification loss to be a ``smooth'' function of nodes over the graph.
When we decrease the graph homophily by randomly adding edges to Cora, this assumption cannot be guaranteed.
Our SGGC does not suffer from this issue because the spectral embedding of ego-graphs is always a ``smooth'' function over graph (see~\cref{sec:theory}).

\scm{SGGC generalizes better than graph condensation and is more efficient.}
Finally, we compare SGGC with graph condensation~\citep{jin2021graph} in terms of the generalizability across GNN architectures and running time.
GCond is model-dependent and generalizes poorly across architectures.
In~\cref{tab:generalization}, we see the performance of graph condensation heavily relies on the GNN architecture used during condensation, while our SGGC is model-agnostic and generalizes well to various types of GNNs.
Although the best performance of graph condensation is comparable to SGGC in ~\cref{tab:generalization}, if we do not tune the architecture for condensation, it is much lower on average.
Specifically, when using SGC for condensation, the test performance of GCond is comparable to SGGC's.
However, when using other architectures, including GCN and SAGE during condensation, the test accuracy of GCond drops for at least 2\% in all settings. 
In terms of running time, apart from the fact that GCond cannot scale to large graphs like ogbn-product, it is much slower than SGGC.
On ogbn-arxiv with the coreset ratio $c/n_t=0.05\%$, graph condensation runs for 494s while SGGC only requires 133s.

\begin{table}[htbp!]
\caption{SGGC generalizes better across GNN architectures than graph condensation (GCond) on Cora with a 50\% ratio.}
\label{tab:generalization}
\centering
\adjustbox{max width=0.75\textwidth}{
\renewcommand{\arraystretch}{1.0} 
{%
\begin{tabular}{ccccc}
        \toprule
        Method & \begin{tabular}[c]{@{}c@{}}Architecture\\used during\\Compression\end{tabular} & GCN & SAGE & SGC \\
        \midrule
        \multirow{3}{*}{GCond} & GCN & 70.6$\pm$3.7 & 60.2$\pm$1.9 & 68.7$\pm$5.4 \\
        & SAGE & 77.0$\pm$0.7 & 76.1$\pm$0.7 & \textbf{77.7$\pm$1.8} \\
        & SGC & \textbf{80.1$\pm$0.6} & \textbf{78.2$\pm$0.9} & \textbf{79.3$\pm$0.7} \\
        \textbf{SGGC (Ours)} & N/A & \textbf{80.2$\pm$1.1} & \textbf{79.1$\pm$0.7} & \textbf{78.5$\pm$1.0} \\
        \bottomrule
    \end{tabular}
}}
\end{table}
\section{Conclusions}
\label{sec:conclusions}

This paper proposes spectral greedy graph coreset (SGGC), a coreset selection method on graph for graph neural networks (GNNs), and node classification.
For the theoretical limitations, we note that the small variation assumption of spectral embeddings on ego graphs may not hold for non-message-passing GNNs and very dense graphs.
For the practical limitations, we address the problem that although SGGC is practically very efficient, similar to most of the coreset algorithms, it has a $O(c n_t n)$ time complexity.
This hinders us from applying a large coreset ratio on very large graphs, which consequently bottlenecks the downstream GNN performance on the coreset graph.
Future work may consider more efficient setups, e.g., online coreset selection and training on graphs with hundreds of millions of nodes.
Considering broader impacts, we view our work mainly as a methodological contribution, which paves the way for more resource-efficient graph representation learning.
Our innovations can enable more scalable ways to do large-network analysis for social good.
However, progress in graph learning might also trigger other hostile social network analyses, e.g., extracting fine-grained user interactions for social tracking.

\section*{Acknowledgement}
Ding and Huang are supported by DARPA Transfer from Imprecise and Abstract Models to Autonomous Technologies (TIAMAT) 80321, National Science Foundation NSF-IIS-2147276 FAI, DOD-ONR-Office of Naval Research under award number N00014-22-1-2335, DOD-AFOSR-Air Force Office of Scientific Research under award number FA9550-23-1-0048, DOD-DARPA-Defense Advanced Research Projects Agency Guaranteeing AI Robustness against Deception (GARD) HR00112020007, Adobe, Capital One and JP Morgan faculty fellowships.

\clearpage
\newpage

\appendix

\section{Proofs and More Theoretical Discussions}
\label{apd:proofs}

In this section, we present the proofs and the complete lemmas or theorems for the theoretical results mentioned in the main paper. We divide this section into three parts; each corresponds to a section/subsection of the paper. We also describe the PCA compression algorithm to the ego-graph's node features at last.

\subsection{Proofs for~\texorpdfstring{\cref{subsec:node-average-coreset}}{Section 3.1}}

We start from the theoretical results in~\cref{subsec:node-average-coreset}. For this subsection, our goal is to show that the spectral embeddings of ego-graphs, as a function of the center node, is a ``smooth'' function on the graph, in the sense that when transformed into the graph spectral domain, their high-frequency components is relatively small.

The above-mentioned ``smoothness'' characterization of ego-graph's spectral embeddings is presented as~\cref{prop:smooth-embedding} in the main paper. In order to prove~\cref{prop:smooth-embedding}, as analyzed in~\cref{sec:theory}, we first need to prove the spectral representation of ego-graph's input features is ``smooth'' (\cref{lem:smooth-input}) under some assumptions, and then show the GNNs we considered (including GCN~\citep{balcilar2021analyzing}) are Lipschitz continuous in the spectral domain (\cref{lem:lipschitz-gnn}, as a corollary of the Theorem 4 in~\citep{balcilar2021analyzing}).

\begin{lemma}[\textbf{Smoothness of the Spectral Representation of Ego-graph's Input Features}]
\label{lem:smooth-input}
For an arbitrary entry ($a$-th row, $b$-th column) of the spectral representations of the ego-graphs' features, $\wt{X}_i=U_{G_i}^\top X_{G_i}$, we denote $\wt{\mathbf{x}}^{(a,b)}=\big[[\wt{X}_1]_{a,b},\ldots,[\wt{X}_n]_{a,b}\big]$. If all node features, i.e., all entries of $X$ are \textit{i.i.d.} unit Gaussians, then $\mathbb{E}[\langle\wt{\mathbf{x}}^{(a,b)}, \egv_i\rangle]=(1-\frac12\lambda_i)^L$.
\end{lemma}
\begin{proof}
For simplicity and without loss of generality, we consider there is only one feature per node, i.e., $X\in\mathbb{R}^{n}$ is a column vector. Thus $\wt{X}_i\in\mathbb{R}^{n}$ is also a column vector. We consider the $k$-th entry of the spectral representation. 

We have,
\begin{equation}
\notag
\begin{split}
&\quad[U_{G_i}^\top X_{G_i}]_k=[U_{G_i}^\top X_{G_i}]_k=\sum_{j\in V_i^L}[U_{G_i}^\top]_{k,j}X_j\\
&= \sum_{j\in[n]}[U^\top]_{k,j}X_j[P^L]_{j,i}\\
&= [U^\top \diag{X_1,\ldots,X_n} U (I-\frac12 \Lambda)^L U^\top]_{k,i}.
\end{split}
\end{equation}
Due to $\mathbb{E}[U^\top \diag{X_1,\ldots,X_n}U]=I$. Thus we have $\mathbb{E}[\wt{\mathbf{x}}^{(k)}] \approx [(I-\frac12 \Lambda)^L U^\top]_{k,:}$, and $\mathbb{E}[\langle\wt{\mathbf{x}}^{(k)}, \egv_i\rangle]=(1-\frac12\lambda_i)^L$.

\end{proof}

\begin{lemma}[\textbf{Lipschitzness of GCN in Spectral Domain}]
\label{lem:lipschitz-gnn}
For an $L$-layer GCN~\citep{kipf2016semi}, if all linear weights $W^{(l)}$ have bounded operator norm, then the corresponding GNN function in the spectral domain $\wt{f}_\theta(\cdot)=U\tran f_\theta(A,U\cdot)$ is Lipschitz continuous.
\end{lemma}
\begin{proof}
For simplicity and without loss of generality, we consider $L=1$. Since $W^{(1)}$ has bounded operator norm, and the non-linear activation function $\sigma(\cdot)$ is Lipschitz continuous, we only need to show the convolution matrix in the spectral domain, i.e., $\wt{C}=U^\top C U$ has bounded operator norm. This directly follows from the Theorem 4 in~\citep{balcilar2021analyzing}, where they show $\|U^\top CU \egv_i\| \approx [1-\bar{d}/(\bar{d}+1)\lambda_i]$.
\end{proof}

Now we prove~\cref{prop:smooth-embedding} as a corollary of~\cref{lem:smooth-input,lem:lipschitz-gnn}.

\begin{proposition}[Smoothness of Spectral Embeddings]
\label{prop:smooth-embedding}
Assuming the node features $X$ are \textit{i.i.d.} Gaussians, and the $L$-layer GNN (\cref{eq:gnn}) have operator-norm bounded linear weights $W^{(l)}$, then the spectral embedding $\wt{Z}_i$ is smooth on graph, in the sense that, $\langle\wt{\mathbf{z}}^{(a,b)}, \egv_i\rangle \leq M\cdot (1-\frac12\lambda_i)^L$ for some constant $M$, where $\wt{\mathbf{z}}^{(a,b)}=\big[[\wt{Z}_1]_{a,b},\ldots,[\wt{Z}_n]_{a,b}\big]$ and $\egv_i$ is the eigenvector corresponding to eigenvalue $\lambda_i$.
\end{proposition}

\begin{prevproof}{prop:smooth-embedding}
Given all the assumptions of~\cref{lem:smooth-input,lem:lipschitz-gnn}, we have that for any entry, $\langle\wt{\mathbf{x}}^{(a,b)}, \egv_i\rangle=(1-\frac12\lambda_i)^L$ in expectation, and $\|\wt{Z}_i-\wt{Z}_j\|=\|U\tran f_\theta(A,U \wt{X}_i)-U\tran f_\theta(A,U \wt{X}_j)\|\leq M\cdot \|\wt{X}_i-\wt{X}_j\|$ for $i,j\in[n]$. This directly leads to that $\langle \wt{\mathbf{z}}^{(a,b)}, \egv_i\rangle\leq M\cdot\langle \wt{\mathbf{x}}^{(a,b)}, \egv_i\rangle = M\cdot(1-\frac12\lambda_i)^L$.
\end{prevproof}

Finally, we show the upper bound on the error of approximating the node-wise average.
\begin{theorem}[\textbf{Upper-bound on the Error Approximating Node-wise Average}]
Under all assumptions of~\cref{prop:smooth-embedding}, we have $\|\sum_{i\in [n_{t}]}w\na_i\cdot \wt{Z}_i - \wt{Z}\|_{F}\leq M\cdot \|P\mathbf{w}\na-\frac{1}{n}\vones\|$ for some constant $M>0$.
\end{theorem}

\begin{proof}
The proof mostly follows from the proof of Theorem 1 in~\citep{linderman2020numerical}. We begin with decomposing an arbitrary $(a,b)$-th entry of $\sum_{i\in [n_{t}]}w\na_i\cdot \wt{Z}_i-\wt{Z}$ in the spectral domain. 
\begin{equation}
\notag
\begin{split}
&\quad\sum_{i\in [n_{t}]}w\na_i\cdot \wt{\mathbf{z}}^{(a,b)}_i-\frac{1}{n}\sum_{j\in[n]}\wt{\mathbf{z}}^{(a,b)}_j\\
&=\sum_{j\in[n]}(\sum_{i\in[n_t]}w\na_i\bm{\delta}_i-\frac{1}{n})\wt{\mathbf{z}}^{(a,b)}_j\\
&=\sum_{k\in[n]}\big\langle\sum_{i\in[n_t]}w\na_i\bm{\delta}_i-\frac{1}{n}, \egv_k\big\rangle \langle\egv_k, {\mathbf{z}}^{(a,b)}\rangle\\
&\leq M'\cdot \sum_{k\in[n]}\big\langle\sum_{i\in[n_t]}w\na_i\bm{\delta}_i-\frac{1}{n}, \egv_k\big\rangle (1-\frac12\lambda_k)^L.
\end{split}
\end{equation}
Thus we have,
\begin{equation}
\notag
\begin{split}
&\quad\big|\sum_{i\in [n_{t}]}w\na_i\cdot \wt{\mathbf{z}}^{(a,b)}_i-\frac{1}{n}\sum_{j\in[n]}\wt{\mathbf{z}}^{(a,b)}_j\big|^2\\
&\leq M'^2\cdot \sum_{k\in[n]}\big|\big\langle\sum_{i\in[n_t]}w\na_i\bm{\delta}_i-\frac{1}{n}, \egv_k\big\rangle\big|^2 (1-\frac12\lambda_k)^{2L}\\
&=M'^2\big\| \sum_{k\in[n]} (1-\frac12\lambda_k)^{L}\big\langle\sum_{i\in[n_t]}w\na_i\bm{\delta}_i-\frac{1}{n}, \egv_k\big\rangle \egv_k \big\|^2\\
&=M'^2\cdot \big\|P^L (\sum_{i\in[n_t]}w\na_i\bm{\delta}_i-\frac{1}{n})\big\|\\
&=M'^2 \cdot \|P\mathbf{w}\na-\frac{1}{n}\vones\|^2.
\end{split}
\end{equation}
Thus we conclude $\|\sum_{i\in [n_{t}]}w\na_i\cdot \wt{Z}_i - \wt{Z}\|_{F}\leq M\cdot \|P\mathbf{w}\na-\frac{1}{n}\vones\|$ where $M=M'\cdot p d$. Here $p$ is the (diffusion) ego-graph size and $d$ is the number of features per node.
\end{proof}

\subsection{Proofs for~\texorpdfstring{\cref{subsec:spectral-classification-coreset}}{Section 3.2}}

The major theoretical result in~\cref{subsec:spectral-classification-coreset} is that the approximation error on the empirical loss of the (spectral) linear classification problem can be upper-bounded by a set function (i.e., a function which only depends on the selected set of nodes), which is formally stated as follows.
\begin{lemma}[\textbf{Upper-bound on the Error Approximating Empirical Loss}]
\label{lem:submodular-error-bound}
For the set function $H(\cdot)$ defined as $H(V)\coloneqq\frac{1}{n}\sum_{i\in[n_t]}\min_{j\in V} \max_{\wt{Z}}|\ell_i(\wt{Z})-\ell_j(\wt{Z})|$, we have \begin{equation}
\notag
\min_{w\lc\in\mathcal{W}}\max_{\wt{Z}}|\sum_{i\in [n_{t}]}w\lc_i\cdot\wt\ell_i(\wt{Z})-\wt{\mathcal{L}}(\wt{Z})|\leq H(V_{w\lc}).
\end{equation}
\end{lemma}

\begin{proof}
The proof follows from Section 3.1 of~\citep{mirzasoleiman2020coresets}. Consider there is a mapping $\zeta: [n]\to V$ and let $w\lc_i$ be the number of nodes mapped to $i\in V$ divided by $n$, i.e., $w\lc_i=\frac{1}{n}\big|\{j\in[n]\mid \zeta{j}=i\}\big|$. Then, $\sum_{i\in [n_{t}]}w\lc_i\cdot\wt\ell_i(\wt{Z})=\sum_{j\in [n]}\wt\ell_{\zeta(i)}(\wt{Z})$ and $|\sum_{i\in [n_{t}]}w\lc_i\cdot\wt\ell_i(\wt{Z})-\wt{\mathcal{L}}(\wt{Z})|\leq \frac{1}{n}\sum_{j\in[n]}|\wt\ell_{\zeta(i)}(\wt{Z})-\wt\ell_i(\wt{Z})|$. From this we can readily derive that, $\min_{w\lc\in\mathcal{W}}|\sum_{i\in [n_{t}]}w\lc_i\cdot\wt\ell_i(\wt{Z})-\wt{\mathcal{L}}(\wt{Z})|\leq$\\$\frac{1}{n}\sum_{i\in[n_t]}\min_{j\in V} |\ell_i(\wt{Z})-\ell_j(\wt{Z})|$, and by taking $\max_{\wt{Z}}$ on both sides we conclude the proof.
\end{proof}

\subsection{Proofs for~\texorpdfstring{\cref{sec:algorithm}}{Section 4}}

The proof for the error-bound on the node-classification loss (\cref{thm:error-gnn-loss}) consists of two parts. First, we want to show that~\cref{alg:sggc} can achieve both objectives (\cref{eq:node-coreset,eq:linear-classification-coreset}) which follows from the Theorem 2 of~\citep{vahidian2020coresets}, \cref{thm:node-error-bound}, and~\cref{lem:submodular-error-bound}. Second, we show under~\cref{assump:bounded-spectral-variance} we can bound the approximation error on the node classification loss.

\begin{prevproof}{thm:error-gnn-loss}
First, using~\cref{alg:sggc} and by the Theorem 2 of~\cite{vahidian2020coresets}, we have $\|P\mathbf{w}\na-\frac{1}{n}\vones\|\leq \sqrt{\frac{1}{n}-(\frac{\kappa}{n})^2}\cdot O((1-\kappa^2\epsilon^2)^{c/2})$ for some $\epsilon>0$ and where $c=|V_w|$ is the size of selected coreset. By~\cref{thm:node-error-bound}, we have the upper-bound on the objective of the node-wise average coreset (\cref{eq:node-coreset}), $\|\sum_{i\in [n_{t}]}w\na_i\cdot \wt{Z}_i - \wt{Z}\|_{F}<M\na$ for some $M\na>0$. Then, by Theorem 1 of~\citep{vahidian2020coresets}, we know $F(V_w)=H(\{i_0\})-H(V_w\cup\{i_0\})\geq O(\frac{1-\kappa}{\sqrt{n}})$. Thus by~\cref{lem:submodular-error-bound}, we have an upper-bound on, $\max_{\wt{Z}}|\ell_i(\wt{Z})-\ell_j(\wt{Z})|\leq M\lc$ for some $M\lc>0$ for any $i\in V_w$ and $j\in[n]$. Note the $w\na$ and $w\lc$ above is the output of~\cref{alg:sggc} and we have $V_{w\na}=V_{w\lc}=V_w$.

Second, by repeatedly using the second bound above, we can get $|\sum_{i\in [n_{t}]}w\na_i w\lc_i\cdot\wt\ell_i(\wt{Z}_i)-\sum_{j\in[n_t]}\sum_{k\in[n]}w\na_j\wt\ell_k(\wt{Z}_j)|\leq M\lc$, where $\zeta: [n]\to V_w$ is the mapping described in~\cref{lem:submodular-error-bound}. While, by the first bound above, \cref{assump:bounded-spectral-variance}, the Jensen-bound and the Lipschitzness of $\wt\ell_k$ (the Lipschitz coefficient is an absolute constant), we can derive, $|\sum_{i\in[n_t]}w\na_i \wt\ell_k(\wt{Z}_i)-\wt\ell_k(\wt{Z}_k)|\leq |\sum_{i\in[n_t]}w\na_i \wt\ell_k(\wt{Z}_i)-\wt\ell_k(\wt{Z})|+|\wt\ell_k(\wt{Z})-\wt\ell_k(\wt{Z}_k)|\leq O(M\na+B\|\wt{Z}\|)$. Combining the two inequalities, we conclude the proof.
\end{prevproof}

\subsection{Compressing Ego-Graph's Node Features via PCA}
In the main paper, we described the algorithm to find the center nodes $V_w$ of the ego-graphs. We then find the union of those ego-graphs, which consists of nodes $V^{(L)}_w=\bigcup_{i\in V_w} V^L_i$ where $V^L_i$ is the set of nodes in the ego-graph centered at node $i$.

We then propose to compress the ego-graph's node features, which consists of $|V^{(L)}_w|d=\xi cd \leq c\times p\times d$ floating point numbers, to have a comparable size with the node features of the center nodes, which consists of only $c\times d$ floating point numbers. Here $c$ is the size of the coreset, $p$ is the diffusion ego-graph size, $d$ is the number of features per node, and $\xi=|V^{(L)}_w|/c\leq p$ because there may be overlaps between the ego-graphs. In practice, the overlaps are often large, and we expect $1<\xi\ll p$.

We first quantize all the nodes' features to half-precision floating-point numbers, and then the desired compress ratio for all the non-center nodes is $1/(\xi-1)$. We then compress the features of each non-center node $j$ according to its shortest path distance to the nearest center node, denoted by $d_{\min}(j)$. Assuming an $L$-layer GNN and thus $L$-depth ego-graphs, for $1\leq l\leq L$, we find the set of nodes $=\{j\in V^{(L)}_w\mid d_{\min}(j)=l\}$, and compress their features by the principal component analysis (PCA) algorithm (which finds a nearly optimal approximation of a singular value decomposition). If we keep the $q_l$ largest eigenvalues, we only need $(|V^{(L,l)}_w|+d+1)q$ half-precision floating point numbers to store the features of $V^{(L,l)}_w$, we find $q_l$ by formula $q_l=\frac{cd}{|V^{(L,l)}_w|+d+1}(\frac{1}{2})^l$, which satisfies the targeted compress ratio $\sum_{l=1}^{l=L}(|V^{(L,l)}_w|+d+1)q_l \leq cd$.
\section{Message-Passing GNNs}
\label{apd:gnns}

In this section, we present more results and discussions regarding the common generalized graph convolution framework.

\scm{Notations.}
Consider a graph with $n$ nodes and $m$ edges. Connectivity is given by the adjacency matrix $A\in\{0,1\}^{n\times n}$ and features are defined on nodes by $X\in\sR^{n\times d}$ with $d$ the length of feature vectors. Given a matrix $C$, let $C_{i,j}$, $C_{i,:}$, and  $C_{:,j}$ denote its $(i,j)$-th entry, $i$-th row, $j$-th column, respectively. Besides, we simplify the notion of $\{1,...,L\}$ to $[L]$. We use $\odot$ to denote the element-wise (Hadamard) product. $\|\cdot\|_p$ denotes the entry-wise $\ell^p$ norm of a vector and $\|\cdot\|_F$ denotes the Frobenius norm. We use $I_n\in\sR^{n\times n}$ to denote the identity matrix, $\vones$ to denote the vector whose entries are all ones, and $\delta_i$ to denote the unit vector in $\sR^{n}$ whose $i$-th entry is $1$. And $\concat$ represents concatenation along the last axis. We use superscripts to refer to different copies of the same kind of variable. For example, $X^{(l)}\in\sR^{n\times f_l}$ denotes node representations on layer $l$. A Graph Neural Network (GNN) layer takes the node representation of a previous layer $X^{(l)}$ as input and produces a new representation $X^{(l+1)}$, where $X=X^{(0)}$ is the input features.

\scm{A common framework for generalized graph convolution.}
GNNs are designed following different guiding principles, including neighborhood aggregation (GraphSAGE~\citep{hamilton2017inductive}, PNA~\citep{corso2020principal}), spatial convolution (GCN~\citep{kipf2016semi}), spectral filtering (ChebNet~\citep{defferrard2016convolutional}, CayleyNet~\citep{levie2018cayleynets}, ARMA~\citep{bianchi2021graph}), self-attention (GAT~\citep{velivckovic2017graph}, Graph Transformers~\citep{yaronlipman2020global, rong2020self, zhang2020graph}), diffusion (GDC~\citep{klicpera2019diffusion}, DCNN~\citep{atwood2016diffusion}), Weisfeiler-Lehman (WL) alignment (GIN~\citep{xu2018powerful}, 3WL-GNNs~\citep{morris2019weisfeiler, maron2019provably}), or other graph algorithms (\citep{xu2019can, loukas2019graph}).  Despite these differences, \emph{nearly all GNNs can be interpreted as performing message passing on node features, followed by feature transformation and an activation function}.
Now we rewrite this expression according to one pointed out by~\citep{balcilar2021analyzing} in the form:
\begin{equation}
\label{eq:gnn-forward}
    X^{(l+1)} = \nonlinear\left(\sum_{r} C^{(r)}X^{(l)}W^{(l,r)}\right),
\end{equation}
where $C^{(r)}\in\sR^{n\times n}$ denotes the $r$-th convolution matrix that defines the message passing operator, $r\in\sZ_+$ denotes the index of convolution, and $\nonlinear(\cdot)$ denotes the non-linearity. $W^{(l,r)}\in\sR^{f_l\times f_{l+1}}$ is the learnable linear weight matrix for the $l$-th layer and $r$-th filter. 

Within this common framework, GNNs differ from each other by the choice of convolution matrices $C^{(r)}$, which can be either fixed or learnable. A learnable convolution matrix relies on the inputs and learnable parameters and can be different in each layer (thus denoted as $C^{(l,r)}$):
\begin{equation}
\label{eq:learnable-conv}
    C^{(l,r)}_{i,j}= \underbrace{\fC^{(r)}_{i,j}}_{\text{fixed}} \cdot \underbrace{h^{(r)}_{\theta^{(l,r)}}(X^{(l)}_{i,:}, X^{(l)}_{j,:})}_{\text{learnable}},
\end{equation}
where $\fC^{(r)}$ denotes the fixed mask of the $r$-th learnable convolution, which may depend on the adjacency matrix $A$ and input edge features $E_{i,j}$. While $h^{(r)}(\cdot, \cdot):\sR^{f_l}\times\sR^{f_l}\to\sR$ can be any learnable model parametrized by $\theta^{(l,r)}$.  We re-formulate some popular GNNs into this generalized graph convolution framework (see~\cref{tab:GNN-models} for more details).

\begin{table*}[htbp!]
\centering
\caption{\label{tab:GNN-models}Summary of GNNs re-formulated as generalized graph convolution.}
\adjustbox{max width=1.0\textwidth}{%
\begin{threeparttable}
\renewcommand*{\arraystretch}{1.4}
\begin{tabular}{Sc Sc Sc Sc Sl} \toprule
Model Name & Design Idea & Conv. Matrix Type & \# of Conv. & Convolution Matrix \\ \midrule
GCN\tnote{1}~~\citep{kipf2016semi} & Spatial Conv. & Fixed & $1$ & $C=\wt{D}^{-1/2}\wt{A}\wt{D}^{-1/2}$ \\
SAGE-Mean\tnote{2}~~\citep{hamilton2017inductive} & Message Passing & Fixed & $2$ & \multicolumn{1}{l}{$\left\{\begin{tabular}[c]{@{}l@{}} $C^{(1)}=I_n$\\ $C^{(2)}=D^{-1}A$ \end{tabular}\right.\kern-\nulldelimiterspace$} \\
GAT\tnote{3}~~\citep{velivckovic2017graph} & Self-Attention & Learnable & \# of heads & \multicolumn{1}{l}{$\left\{\begin{tabular}[c]{@{}l@{}} $\fC^{(r)}=A+I_n$\itand\\ $h^{(r)}_{\va^{(l,r)}}(X^{(l)}_{i,:}, X^{(l)}_{j,:})=\exp\big(\mathrm{LeakyReLU}($\\ \quad$(X^{(l)}_{i,:}W^{(l,r)}\concat X^{(l)}_{j,:}W^{(l,r)})\cdot\va^{(l,r)})\big)$ \end{tabular}\right.\kern-\nulldelimiterspace$} \\ \bottomrule
\end{tabular}
\begin{tablenotes}[para]
    \item[1] Where $\wt{A}=A+I_n$, $\wt{D}=D+I_n$.
    \item[2] $C^{(2)}$ represents mean aggregator. Weight matrix in~\citep{hamilton2017inductive} is $W^{(l)}=W^{(l,1)}\concat W^{(l,2)}$.
    \item[3] Need row-wise normalization. $C^{(l,r)}_{i,j}$ is non-zero if and only if $A_{i,j}=1$, thus GAT follows direct-neighbor aggregation.
\end{tablenotes}
\end{threeparttable}}
\end{table*}

Most GNNs can be interpreted as performing message passing on node features, followed by feature transformation and an activation function, which is known as the common ``generalized graph convolution'' framework. 

\scm{GNNs that cannot be defined as graph convolution.} Some GNNs, including Gated Graph Neural Networks~\citep{li2015gated} and ARMA Spectral Convolution Networks~\citep{bianchi2021graph} cannot be re-formulated into this common graph convolution framework because they rely on either Recurrent Neural Networks (RNNs) or some iterative processes, which are out of the paradigm of message passing.

\section{More Related Work}
\label{apd:related}

\wcm{Dataset condensation} (or distillation) is first proposed in~\citep{wang2018dataset} as a learning-to-learn problem by formulating the network parameters as a function of synthetic data and learning them through the network parameters to minimize the training loss over the original data. However, the nested-loop optimization precludes it from scaling up to large-scale in-the-wild datasets. \citep{zhao2020dataset} alleviate this issue by enforcing the gradients of the synthetic samples w.r.t. the network weights to approach those of the original data, which successfully alleviates the expensive unrolling of the computational graph. Based on the meta-learning formulation in~\citep{wang2018dataset}, \citep{bohdal2020flexible} and \citep{nguyen2020dataset,nguyen2021dataset} propose to simplify the inner-loop optimization of a classification model by training with ridge regression which has a closed-form solution, while \citep{such2020generative} model the synthetic data using a generative network. To improve the data efficiency of synthetic samples in the gradient-matching algorithm, \citep{zhao2021dataseta} apply differentiable Siamese augmentation, and~\citep{kim2022dataset} introduce efficient synthetic-data parametrization. Recently, a new distribution-matching framework~\citep{zhao2021datasetb} proposes to match the hidden features rather than the gradients for fast optimization but may suffer from performance degradation compared to gradient-matching~\citep{zhao2021datasetb}, where~\citep{kim2022dataset} provide some interpretation.

Recent benchmark~\citep{guo2022deepcore} of \wcm{model-based coreset methods} on image classification indicates \textit{Forgetting} and \textit{GraNd} are among the best-performing ones but still evidently underperform the dataset condensation approach (see~\cref{apd:related}).
\textit{Forgetting}~\citep{toneva2018empirical} measures the forgetfulness of trained samples and drops those that are not easy to forget. 
\textit{GraNd}~\citep{paul2021deep} selects the training samples that contribute most to the training loss in the first few epochs.

\wcm{Graph sampling} methods~\citep{chiang2019cluster,zeng2019graphsaint} can be as simple as uniformly sampling a set of nodes and finding their induced subgraph, which is understood as a graph-counterpart of uniform sampling of \textit{i.i.d.} samples. However, most of the present graph sampling algorithms (e.g., ClusterGCN~\citep{chiang2019cluster} and GraphSAINT~\citep{zeng2019graphsaint}) are designed for sampling multiple subgraphs (mini-batches), which form a cover of the original graph for training GNNs with memory constraint. Therefore, those graph mini-batch sampling algorithms are effective graph partitioning algorithms and not optimized to find just one representative subgraph.
\section{Implementation Details}
\label{apd:implement}

\scm{Packages and Hardware specs.}
Generally, our project is developed upon the Pytorch framework, and we use Pytorch Geometric (\url{https://pytorch-geometric.readthedocs.io/en/latest/}) to acquire datasets Cora, Citeseer, Pubmed, and Flickr, and utilize Ogb (\url{https://ogb.stanford.edu/}) to get dataset Arxiv. The coreset methods on the graph are implemented based on  \citet{guo2022deepcore} (\url{https://github.com/patrickzh/deepcore}) and our downstream GNN structures, GCN, GraphSage, and SGC, are implemented based on \citet{jin2021graph} (\url{https://github.com/ChandlerBang/GCond}). The experiments are conducted on hardware with Nvidia GeForce RTX 2080 Ti(11GB GPU).

\scm{Dataset statistics.}
We adopt five graph datasets, Cora, Citeseer, Pubmed, Flickr, and Arxiv in our experiments. Cora, Citeseer, and Pubmed are citation datasets whose node features represent the most common words in the text, and two nodes, each representing a paper, are connected if either one cites the other. Flickr is an image network where images connect if they share common properties, such as similar figures or buildings. While Arxiv is a directed citation network that contains all computer science papers in Arxiv indexed by MAG\footnote{https://www.microsoft.com/en-us/research/project/microsoft-academic-graph/}. A directed edge from paper A to paper B establishes if A cites B. Here is detailed information on the datasets.

\begin{table*}[htbp!]
\centering 
\caption{\label{tab:datainfo} Statistics of the datasets.}
\begin{tabular}{llllllll}
\hline
           & Nodes  & Edges                                                                      & Features & Classes & Train (\%) & Validation (\%) & Test (\%) \\ \hline
Cora       & 2,708   & 10,556                                                                      & 1,433       & 7     & 5.2       & 18.5  & 76.3 \\
Citeseer   & 3,327   & 9,104                                                                       & 3,703       & 6     & 3.6       & 15.0  & 81.4 \\
Pubmed     & 19,717  & 88,648                                                                      & 500        & 3     & 0.3       & 2.5   & 97.2 \\
Flickr     & 89,250  & 899,756                                                                     & 500        & 7     & 50.0      & 25.0  & 25.0 \\
Arxiv & 169,343 & \multicolumn{1}{r}{\cellcolor[HTML]{F5F5F5}{\color[HTML]{212121} 1,166,243}} & 128        & 40    & 53.7      & 17.6  & 28.7 \\ \hline
\end{tabular}
\end{table*}

From \cref{tab:datainfo}, we could see that Cora, Citeseer, and Pubmed are network data with nodes of no more than 20,000 and edges of less than 100,000, which can be deemed as small datasets. While Flickr (which has 89,250 nodes and 899,756 edges) and Arxiv (which has 169,343 nodes and 1,166,243 directed edges) are much larger datasets, and more than 50\% of the nodes are training nodes, so testing GNN accuracy for model-based coreset methods could be both time and space consuming.

\scm{Hyper-parameter setups of SGGC.}
For SGGC, there are three hyper-parameters: slack parameter (denoted as $\kappa$), max budget (denoted as $s$), and diffusion ego-graph size. 
\begin{enumerate*}[label=(\arabic*)]
    \item Slack parameter, denoted by $\kappa$, means we could pick the nodes whose alignment score has at least $\kappa$ of the highest alignment score of $v*$ in \cref{alg:sggc}. The $\kappa$ shows the balance of our algorithm between SCGIGA and CraigLinear. When $\kappa$ is 1, we pick nodes in the same way as SCGIGA. When we select set $\kappa$ to be 0, this algorithm ignores geodesic alignment and becomes the CraigLinear algorithm. Therefore, $\kappa$ determines which algorithm in SCGIGA and CraigLinear SGGC is more similar.
    \item Max budget, denoted by $s$, is the maximum number of nodes we could pick for each epoch in line 6 of \cref{alg:sggc}. This procedure could help accelerate the selection process for large graphs such as Flickr and Arxiv. We aim to find the largest max budget, which could remain an unnoticeable drop in accuracy compared to the case when the max budget is 1.
    \item Diffusion ego graph size, denoted by $p$, is the ego-graph size we fix for every coreset node in the ego-graph selecting procedure. Since the diffusion ego-graph size of one node is smaller than its two-degree standard ego-graph, the GNN test performance becomes better when the diffusion ego-graph is larger. However, we need to control the size of the diffusion ego-graph so that the induced subgraph has comparable memory with that of the node-wise coreset selection methods. Therefore, our ablation experiment in this part aims to find the smallest ego-graph size for each dataset where test performance does not increase anymore.
\end{enumerate*}

So, according to our analysis of hyper-parameters, we design our hyper-parameter selecting strategy and list the selected hyper-parameters below.

\begin{table*}[htbp!]
\centering
\caption{\label{tab:HPselection}SGGC hyper-parameter setups}
\begin{tabular}{lccc}
\hline
         & Slack Parameter ($\kappa$) & \multicolumn{1}{l}{Max Budget ($s$)} & \multicolumn{1}{l}{Diffusion Ego-Graph Size ($p$)} \\ \hline
Cora     & 0.999    & 1                               & 16                            \\
Citeseer & 0.5      & 1                               & 8                             \\
Pubmed   & 0.1      & 1                               & 16                            \\
Flickr   & 0.5      & 10                              & 8                             \\
Arxiv    & 0.1      & 10                              & 8                             \\ 
Products & 0.1      & 8                               & 2                             \\                                             
`Reddit2  & 0.1      & 8                               & 8                             \\ \hline
\end{tabular}
\end{table*}

In the hyper-parameter tuning process, we first determine the diffusion ego-graph size for different datasets under the fixed max budget, the slack parameter $kappa$, and the fraction ratio since this hyper-parameter does not significantly influence the experiment result. Hence, we could pick the smallest ego-graph size with which the model reaches the highest GNN test accuracy. Taking Cora and Citeseer as an example, the test accuracy becomes around the best when the ego-graph size is 16 for Cora and 8 for Citeseer. After that, we fix the ego-graph size as we just selected and then choose the slack parameter $\kappa$ for each dataset according to GNN test accuracy. For instance, according to~\cref{tab:kappa}, the best performance achieves when $\kappa=0.999$ for Cora and $\kappa=0.5$ for Citeseer. Finally, we fix the selected diffusion ego-graph size and slack parameter $\kappa$ to find the best max budget $s$. This step is only for Flickr and Arxiv. We try to increase the max budget from 1, and the aim is to find the largest max budget which does not sacrifice noticeable accuracy to accelerate the SGGC algorithm as much as possible. Therefore, based on~\cref{tab:budget}, we observe that the accuracy drops out of the highest performance error bar when the max budget is larger than 10 for Flickr and Arxiv, so we choose the max budget as 10 for both of them.

\textbf{Selection strategy setups}
The GNN training varies according to different selection strategies (node-wise selection, standard ego-graph selection, and diffusion ego-graph selection). 
\begin{enumerate*}[label=(\arabic*)]
\item When we choose node-wise selection, we directly induce a subgraph composed of the coreset nodes and train GNN inductively. 
\item For the standard ego-graph strategy, we union the 2-hop ego-graphs of every selected coreset node and use this node and the labeled coreset nodes to train GNN transductively. 
\item The diffusion ego-graph selection strategy is slightly different from the standard one, which controls the ego-graph size $p$. If the size of the 2-degree ego graph of one node exceeds $p$, then we randomly cut the ego-graph size to $p$. If not, we try to increase the ego-graph degree to add more nodes until the node has an ego-graph size $p$. Here we point out that not all ego-graphs could reach the size $p$ because the connection component the node is in has a size smaller than $p$.
\end{enumerate*}

\textbf{Hyper-parameter setups of other model-agnostic coreset methods.}
The model-agnostic coreset methods, Uniform, Herding, kCenterGreedy, and Cal, are determined algorithms that directly select coresets based on node features, thus are hyper-parameter free. Besides, it is also worth noting that those coreset selection methods do not utilize the graph structure information, i.e., the graph adjacency matrix. For details, see \citet{guo2022deepcore}

\textbf{Hyper-parameter setups of model-based coreset methods.}
The model-based method takes advantage of the node's gradient/hidden features in GNN to select the coresets. Our experiments are conducted on a pre-trained 2-layer GCN model, specifically GCNConv (feature dim.,256)-50\%dropout-GCNConv (256, class num.)-log softmax. This model is trained for five epochs. The optimizer is Adam, with a learning rate of 0.01 and weight decay 5e-4. Besides, we also offer other pre-trained models, such as SGC and SAGE, in our released code repository. For Craig, we choose the submodular function to be the Facility Location. For GradMatch, we take the regularization parameter in orthogonal matching pursuit to be one, following the default setting of \citet{guo2022deepcore}. Forgetting, Craig, Glister, and GraNd are hyper-parameter-free except for the pretraining part.   

\textbf{Downstream GNN architectures.}
We implement GCN, SGC, and GraphSage as the downstream GNN architecture. GCN has two convolution layers and a 50\% dropout layer between them. SGC and GraphSage are both 2-layer without dropout. Those three models do not have batch normalization. The optimizers for the three models are all Adam equipped with a learning rate of 0.01 and weight decay of 5e-4. The hidden dimension is 256 in all three models. For every coreset selection of each dataset, we train the GNN 10 times, and the GNN is trained for 600 epochs every time. 
\section{More Experimental Results.}
\label{apd:experiments}

\subsection{Training Time with SGGC}
As a coreset selection method, SGGC can significantly reduce the training time and memory complexities of GNNs on large graphs. Specifically, when SGGC identifies a coreset of $c$ nodes for a graph with $n_t$ training nodes (usually $c \ll n_t$), training GNNs on the SGGC coreset graph exhibits linear time and memory complexity to $c$ rather than $n_t$. In other words, SGGC achieves sublinear training time and memory complexities (to the original graph size) simultaneously. In~\cref{tab:training-time}, we have recorded the actual training time and memory of GNNs on the SGGC coreset graph and the original graph under different combinations of datasets and coreset ratios ($c/n_t$). 2-Layer GCN with 256 hidden dimensions is trained for 200 epochs. Our experimental results verify the sublinear complexities of training GNNs using SGGC.

\begin{table*}[!htbp]
\centering
\caption{\label{tab:training-time}GNN training time on SGGC coreset graphs and original graphs.}
\adjustbox{max width=1.0\textwidth}{%
\renewcommand*{\arraystretch}{1.2}
\begin{tabular}{llllll}
\hline
    Dataset & Ratio & \begin{tabular}[l]{@{}l@{}}Training Time on\\ SGGC Coreset Graph\end{tabular} & \begin{tabular}[l]{@{}l@{}}Training Time on\\ Original Graph\end{tabular} & \begin{tabular}[l]{@{}l@{}}Training Memory on\\ SGGC Coreset Graph\end{tabular} & \begin{tabular}[l]{@{}l@{}}Training Memory on\\ Original Graph\end{tabular} \\ \hline
    Flickr & 0.2\% & 9.9 s & 45.4 s & 2.8 MB & 147.5 MB \\ 
    Flickr & 1.0\% & 10.7 s & 45.4 s & 3.1 MB & 147.5 MB \\ 
    ogbn-arxiv & 0.5\% & 23.4 s & 68.1 s & 17.1 MB & 963.1 MB \\ 
    ogbn-arxiv & 1.0\% & 28.0 s & 68.1 s & 4.56 MB & 963.1 MB \\ \hline
\end{tabular}
}
\end{table*}

\subsection{SGGC with JKNets and Graph Transformers}
We argue that SGGC is also applicable to GNNs with a large or even non-local receptive field. We address this argument for each assumption, respectively, with experimental support.

We tested our SGGC with JKNets~\citep{xu2018representation} on ogbn-arxiv and Reddit. Many complex model-based coreset algorithms run out of time/out of memory on ogbn-arxiv and Reddit, like in~\cref{tab:compare}. On ogbn-arxiv, we use a 4-layer GCN-based JK-MaxPool. The other hyperparameters are the same as in~\cref{tab:compare}. The test accuracy of JKNets trained on the coreset graphs obtained by different algorithms is in~\cref{tab:global-gnns-jknet}. We see SGGC is still better than the other baselines.

\begin{table*}[!htbp]
\centering
\caption{\label{tab:global-gnns-jknet}Performance results of applying SGGC with JKNets.}
\adjustbox{max width=1.0\textwidth}{%
\renewcommand*{\arraystretch}{1.2}
\begin{tabular}{llllllll}
\hline
    Dataset & Ratio & Uniform & Herding & k-Center & Forgetting & SGGC (Ours) & Full Graph (Oracle) \\ \hline
    ogbn-arxiv & 0.5\% & 55.0 $\pm$ 1.1 & 35.7 $\pm$ 0.9 & 49.1 $\pm$ 1.8 & 51.5 $\pm$ 1.9 & 56.2 $\pm$ 0.7 & 72.2 $\pm$ 0.3 \\ 
    ogbn-arxiv & 1.0\% & 59.8 $\pm$ 1.5 & 40.2 $\pm$ 2.2 & 53.6 $\pm$ 0.6 & 59.5 $\pm$ 1.6 & 60.5 $\pm$ 0.8 & 72.2 $\pm$ 0.3 \\ 
    Reddit & 0.1\% & 20.3 $\pm$ 6.5 & 13.7 $\pm$ 2.5 & 25.8 $\pm$ 1.6 & 19.8 $\pm$ 5.5 & 40.4 $\pm$ 1.9 & 96.5 $\pm$ 0.8 \\ 
    Reddit & 0.2\% & 32.3 $\pm$ 6.1 & 14.3 $\pm$ 1.3 & 28.0 $\pm$ 2.7 & 22.5 $\pm$ 2.1 & 42.8 $\pm$ 1.0 & 96.5 $\pm$ 0.8 \\ \hline
\end{tabular}
}
\end{table*}

We also tested SGGC with Graph Transformers~\citep{shi2020masked} on ogbn-arxiv. The other hyperparameters are the same as in~\cref{tab:compare}. The test accuracy of Graph Transformers trained on the coreset graphs obtained by different algorithms is as follows. We see SGGC is still better than the other baselines in most cases.

\begin{table*}[!htbp]
\centering
\caption{\label{tab:global-gnns-transformer}Performance results of applying SGGC with Graph Transformer.}
\adjustbox{max width=1.0\textwidth}{%
\renewcommand*{\arraystretch}{1.2}
\begin{tabular}{llllllll}
\hline
    Dataset & Ratio & Uniform & Herding & k-Center & Forgetting & SGGC (Ours) & Full Graph (Oracle) \\ \hline
    ogbn-arxiv & 0.5\% & 35.2 $\pm$ 2.9 & 35.2 $\pm$ 0.6 & 47.6 $\pm$ 1.6 & 52.2 $\pm$ 2.0 & 52.9 $\pm$ 1.3 & 72.1 $\pm$ 0.4 \\ 
    ogbn-arxiv & 1.0\% & 58.3 $\pm$ 1.2 & 38.9 $\pm$ 2.5 & 51.0 $\pm$ 1.9 & 55.9 $\pm$ 2.7 & 56.1 $\pm$ 2.3 & 72.1 $\pm$ 0.4 \\ \hline
\end{tabular}
}
\end{table*}

\subsection{SGGC on Low-Homophily Graphs and Graphs requiring Long-Range Reasoning}
We conduct experiments on two real-world low-homophily graphs, Chameleon (homophily ratio $h=0.23$) and Squirrel (homophily ratio $h=0.22$)~\citep{rozemberczki2021multi}. The homophily ratio is defined as,
\begin{equation}
\notag
h=\frac{1}{|E|} \sum_{(i,j)\in E} \mathbb{1}\{y_j=y_k\},
\end{equation}
where $E$ is the set of edges, $y_i$ is the label of node $i$, and $\mathbb{1}\{\cdot\}$ is the indicator function. Generally, the homophily ratio $h$ describes how likely the two end nodes of an edge are in the same class. Common node classification benchmarks are often high-homophily; for example, the homophily ratio of Cora is around $0.81$. The Chameleon and Squirrel graphs are relatively small, allowing us to execute most baseline methods on them successfully. As shown in~\cref{tab:low-homo}, we find our SGGC nearly always shows the best performance, which indicates SGGC is robust to and suitable for low-homophily graphs.

\begin{table*}[!htbp]
\centering
\caption{\label{tab:low-homo}Performance comparison low-homophily graphs.}
\adjustbox{max width=1.0\textwidth}{%
{\Huge
\renewcommand*{\arraystretch}{1.4}
\begin{tabular}{lllllllllllll}
\hline
    Dataset & Ratio & Uniform & Herding & K-Center & CRAIG & Forgetting & Glister & GradMatch & GraNd & Cal & SGGC (ours) & Full Graph (Oracle) \\ \hline
    Chameleon & 12.5\% & 49.0 $\pm$ 2.8 & 46.3 $\pm$ 0.8 & 44.7 $\pm$ 1.1 & 49.2 $\pm$ 3.8 & 46.9 $\pm$ 2.3 & 50.7 $\pm$ 2.2 & 43.2 $\pm$ 1.6 & 42.5 $\pm$ 1.9 & 49.5 $\pm$ 1.0 & 48.1 $\pm$ 2.2 & 58.7 $\pm$ 1.6 \\ 
    Chameleon & 25.0\% & 39.7 $\pm$ 4.0 & 32.9 $\pm$ 1.8 & 29.9 $\pm$ 1.7 & 38.2 $\pm$ 4.1 & 32.4 $\pm$ 3.2 & 38.8 $\pm$ 3.0 & 33.5 $\pm$ 2.9 & 37.3 $\pm$ 2.9 & 40.8 $\pm$ 2.9 & 40.1 $\pm$ 2.3 & 58.7 $\pm$ 1.6 \\ 
    Chameleon & 50.0\% & 50.4 $\pm$ 3.8 & 46.9 $\pm$ 1.9 & 45.3 $\pm$ 1.2 & 50.2 $\pm$ 2.5 & 47.1 $\pm$ 1.5 & 50.8 $\pm$ 2.8 & 43.6 $\pm$ 1.1 & 43.2 $\pm$ 1.4 & 48.9 $\pm$ 2.1 & 49.4 $\pm$ 1.9 & 58.7 $\pm$ 1.6 \\ 
    Squirrel & 12.5\% & 39.2 $\pm$ 1.0 & 38.5 $\pm$ 0.9 & 38.0 $\pm$ 0.8 & 38.6 $\pm$ 1.3 & 35.6 $\pm$ 0.5 & 38.4 $\pm$ 1.3 & 36.8 $\pm$ 0.7 & 36.5 $\pm$ 0.8 & 38.0 $\pm$ 0.5 & 38.7 $\pm$ 1.3 & 44.5 $\pm$ 0.4 \\ 
    Squirrel & 25.0\% & 33.3 $\pm$ 1.5 & 31.6 $\pm$ 0.7 & 31.7 $\pm$ 0.8 & 33.3 $\pm$ 1.7 & 31.6 $\pm$ 0.8 & 33.3 $\pm$ 1.2 & 31.0 $\pm$ 0.6 & 30.4 $\pm$ 1.0 & 32.7 $\pm$ 0.5 & 34.1 $\pm$ 1.2 & 44.5 $\pm$ 0.4 \\ 
    Squirrel & 50.0\% & 37.4 $\pm$ 1.2 & 36.9 $\pm$ 1.2 & 36.4 $\pm$ 1.8 & 37.9 $\pm$ 1.8 & 35.8 $\pm$ 0.6 & 37.8 $\pm$ 1.0 & 34.5 $\pm$ 1.0 & 35.2 $\pm$ 0.7 & 36.5 $\pm$ 0.5 & 38.2 $\pm$ 1.1 & 44.3 $\pm$ 0.4 \\ \hline
\end{tabular}
}}
\end{table*}

We now include comparison results between SGGC and coreset baselines on the PascalVOC-SP dataset~\citep{dwivedi2022long}. PascalVOC-SP is a node classification dataset that requires long-range interaction reasoning in GNNs to achieve strong performance in a given task. With 5.4 million nodes, PascalVOC-SP is larger than all datasets in~\cref{tab:compare}, and many complex model-based coreset algorithms run out of time or memory. The performance metric for PascalVOC is macro F1, and our SGGC consistently shows the best performance, as shown in~\cref{tab:long-range}. This indicates that SGGC is also applicable to node classification tasks that require long-range information.

\begin{table*}[!htbp]
\centering
\caption{\label{tab:long-range}Performance comparison on PascalVOC-SP, a graph dataset that requires long-range interaction reasoning.}
\adjustbox{max width=\textwidth}{%
{\Huge
\renewcommand*{\arraystretch}{1.2}
\begin{tabular}{llllllll}
\hline
Dataset & Ratio & Uniform & Herding & k-Center & Forgetting & SGGC (Ours) & Full Graph (Oracle) \\
\hline PascalVOC-SP & $0.05 \%$ & $0.060 \pm 0.015$ & $0.040 \pm 0.006$ & $0.050 \pm 0.013$ & $0.044 \pm 0.009$ & $0.069 \pm 0.011$ & $0.263 \pm 0.006$ \\
PascalVOC-SP & $0.10 \%$ & $0.073 \pm 0.018$ & $0.051 \pm 0.004$ & $0.068 \pm 0.009$ & $0.062 \pm 0.016$ & $0.080 \pm 0.008$ & $0.263 \pm 0.006$ \\ \hline
\end{tabular}
}}
\end{table*}

\subsection{Ablation Studies}
As the ablation studies, we first discuss how three key factors, max budget, diffusion ego-graph size, and slack parameter(denoted as $\kappa$), affect the GNN test accuracy of our SGGC algorithm. We perform experiments on GCN and display the SGGC performance under different hyper-parameters on Citeseer, Cora, Flickr, and Arxiv. Then we explore whether different selection strategies affect GNN test accuracy on large graphs, Flickr, and Arxiv.

\scm{Ego-graph size.}
Diffusion ego graph size is an important hyper-parameter to determine in coreset selection. Given the diffusion ego-graph size of one node is smaller than its two-degree standard ego-graph, when the diffusion ego-graph is larger, the GNN test performance becomes better. However, we need to control the size of the diffusion ego-graph so that the induced subgraph has a comparable memory with that of the node-wise coreset selection methods. Therefore, our ablation experiment in this part aims to find the smallest ego-graph size for each dataset where test performance does not increase anymore. 

\begin{table*}[htbp!]
\centering
\caption{\label{tab:ego}GNN test performance with different diffusion ego-graph size on Flickr and Arxiv}
\adjustbox{max width=1.0\textwidth}{%
\renewcommand*{\arraystretch}{1.2}
\begin{tabular}{llllllll}
\hline
Diffusion Ego-Graph Size ($p$)         & 4            & 8            & 12           & 16           & 20           & 24           & 28           \\ \hline
Cora     & 80.4$\pm$0.7 & 80.1$\pm$0.5 & 79.9$\pm$0.8 & 80.5$\pm$0.7 & 80.6$\pm$0.9 & 80.0$\pm$0.6 & 80.1$\pm$0.7 \\
Citeseer & 67.2$\pm$1.1 & 67.7$\pm$0.6 & 67.9$\pm$0.8 & 67.2$\pm$0.5 & 67.2$\pm$1.1 & 67.2$\pm$1.1 & 66.5$\pm$1.2 \\ \hline
\end{tabular}
}
\end{table*}

The diffusion ego-graph size on Cora and Citeseer seems to have no clear influence on test accuracy if the ego-graph size is larger than a very small positive integer. 

\scm{Slack parameter.}
In \cref{alg:sggc}, slack parameter (i.e.$\kappa$) means we could pick the nodes whose alignment score has at least $\kappa$ of the highest alignment score of $v*$. This $\kappa$ shows the balance of our algorithm between SCGIGA and CraigLinear. When $\kappa$ is 1, we pick nodes like pure SCGIGA. When we select set $\kappa$ to be 0, this algorithm ignores geodesic alignment and becomes the CraigLinear algorithm. Therefore, $\kappa$ determines which algorithm in SCGIGA and CraigLinear SGGC is more similar.  

\begin{table*}[htbp!]
\centering
\caption{\label{tab:kappa}GNN test performance with different kappa on Cora and Citeseer}
\adjustbox{max width=1.0\textwidth}{%
\renewcommand*{\arraystretch}{1.2}
\begin{tabular}{lllllll}
\hline
Slack Parameter ($\kappa$)             & \multicolumn{1}{c}{\begin{tabular}[c]{@{}c@{}}CraigLinear\\ ($\kappa$=0)\end{tabular}} & \multicolumn{1}{c}{0.001} & \multicolumn{1}{c}{0.1}   & \multicolumn{1}{c}{0.25} & \multicolumn{1}{c}{0.375} & \multicolumn{1}{c}{0.5} \\ \hline
Cora     & 78.1$\pm$1.1                                                                           & 76.2$\pm$1.2              & 75.8$\pm$1.6              & 76.1$\pm$0.9             & 76.6$\pm$1.2              & 76.7$\pm$2.0            \\
Citeseer & 66.6$\pm$1.9                                                                           & 63.1$\pm$6.6              & 67.8$\pm$1.8              & 69.1$\pm$1.7             & 69.7$\pm$2.3              & 69.8$\pm$1.5            \\ \hline
\end{tabular}
}
\end{table*}

\begin{table*}[htbp!]
\adjustbox{max width=1.0\textwidth}{%
\renewcommand*{\arraystretch}{1.2}
\begin{tabular}{lllllll}
\hline
Slack Parameter ($\kappa$)   & \multicolumn{1}{c}{0.625} & \multicolumn{1}{c}{0.75} & \multicolumn{1}{c}{0.9}    & \multicolumn{1}{c}{0.999} & \multicolumn{1}{c}{0.9999} & \multicolumn{1}{c}{\begin{tabular}[c]{@{}c@{}}SCGIGA\\ ($\kappa$=1)\end{tabular}} \\ \hline
    Cora     & 76.9$\pm$1.5              & 77.9$\pm$0.9             & 75.0$\pm$1.3               & 80.4$\pm$0.8              & 80.0$\pm$0.7               & 78.3$\pm$1.3                                                                      \\
    Citeseer & 69.2$\pm$1.8              & 69.0$\pm$1.5             & 68.4$\pm$2.6               & 69.1$\pm$0.9              & 68.2$\pm$0.5               & 66.9$\pm$1.1                                                                      \\ \hline
\end{tabular}
}
\end{table*}

From \cref{tab:kappa}, the test performance keeps comparable from around 0.001 to at least as high as 0.75 for Cora. When $\kappa$ approximates 1, test performance reaches its highest. However, for Citeseer, the accuracy is the highest at around 0.5. This result shows that $\kappa$ selection highly varies according to datasets. It is interesting that if we plot the $\kappa$-accuracy graph, the curve drops dramatically at 0 and 1. This remains not clear why the result would change that dramatically.

\scm{Max budget.}
Max budget is the number of nodes we could pick for each epoch in line 6 of \cref{alg:sggc}. This procedure could help accelerate the selection process for large graphs such as Flickr and Arxiv. We aim to find the largest max budget which could remain an unnoticeable drop in accuracy compared to the case when the max budget is 1.

\begin{table*}[htbp!]
\centering
\caption{\label{tab:budget}GNN test performance with different max budget on Flickr and Arxiv}
\adjustbox{max width=1.0\textwidth}{%
\renewcommand*{\arraystretch}{1.2}
\begin{tabular}{llllllll}
\hline
Max Budget ($s$)       & \multicolumn{1}{c}{1} & \multicolumn{1}{c}{5} & \multicolumn{1}{c}{10} & \multicolumn{1}{c}{15} & \multicolumn{1}{c}{20} & \multicolumn{1}{c}{25} & \multicolumn{1}{c}{30} \\ \hline
Flickr & 49.5$\pm$0.2          & 49.3$\pm$0.4          & 49.3$\pm$0.3           & 49.0$\pm$0.5           & 49.1$\pm$0.4           & 48.7$\pm$0.4           & 48.4$\pm$0.8           \\
Arxiv  & 60.9$\pm$1.5          & 61.4$\pm$1.4          & 59.8$\pm$1.3           & 56.9$\pm$2.0           & 57.7$\pm$3.2           & 54.8$\pm$3.4           & 50.2$\pm$5.6           \\ \hline
\end{tabular}
}
\end{table*}

From this table, we could see that increasing the max budget does not affect test accuracy much before the budget reach as large as 10, so we could take the budget to 10 for both Flickr and Arxiv.

\scm{Standard ego graph sizes}
Recall that in \cref{sec:experiments}, our SGGC adopts the diffusion ego-graph selection strategy because it could dramatically save memory storage by cutting the ego-graph size of each node while keeping the GNN test accuracy from a noticeable decrease. To see it more clearly, we compare diffusion ego-graph size and standard average ego-graph size for all datasets.

\begin{table*}[htbp!]
\centering
\caption{\label{tab:diffVSstan}Ego-graph size for different selection strategy on datasets.}
\adjustbox{max width=1.0\textwidth}{%
\renewcommand*{\arraystretch}{1.2}
\begin{tabular}{llllll}
\hline
                         & Cora & Citeseer & Pubmed & Flickr & Arxiv  \\ \hline
Diffusion ego-graph size & 16   & 8        & 16     & 8      & 8      \\
Standard ego-graph size  & 36.8 & 15.1     & 60.1   & 875.7  & 3485.2 \\ \hline
\end{tabular}
}
\end{table*}

\subsection{More results on large graphs}

As a good continuation of \cref{tab:compare}, we conduct main coreset selection methods on large graphs, Flickr and Arxiv, to further understand the power of our SGGC.

\begin{table*}[htbp!]
\centering
\caption{\label{tab:lgfration} GNN test performance on large graph with different fraction ratio}
\adjustbox{max width=1.0\textwidth}{%
\renewcommand*{\arraystretch}{1.2}
\begin{tabular}{cllllllll}
\hline
\multicolumn{1}{l}{}     &       & Uniform      & Herding      & kCenterGreedy & Forgetting   & Glister      & GradMatch    & SGGC         \\ \hline
\multirow{2}{*}{Flickr} & 0.2\% & 45.9$\pm$2.0 & 44.5$\pm$0.9 & 46.6$\pm$1.8  & 45.7$\pm$2.1 & 47.4$\pm$0.8 & 48.3$\pm$0.5 & 46.9$\pm$1.5 \\
                         & 1\%   & 47.4$\pm$1.8 & 46.3$\pm$0.5 & 46.7$\pm$0.9  & 47.4$\pm$1.9 & 48.4$\pm$1.0 & 48.5$\pm$0.6 & 48.6$\pm$0.7 \\ \hline
\multirow{2}{*}{Arxiv}   & 0.1\% & 47.2$\pm$6.2 & 33.8$\pm$1.3 & 41.5$\pm$5.8  & 43.6$\pm$6.1 & 41.5$\pm$8.6 & 45.6$\pm$3.9 & 41.1$\pm$6.8 \\
                         & 0.5\% & 57.4$\pm$2.6 & 45.6$\pm$0.4 & 57.4$\pm$1.0  & 56.6$\pm$2.7 & 56.1$\pm$3.0 & 51.2$\pm$4.2 & 60.0$\pm$1.9 \\ \hline
\end{tabular}
}
\end{table*}
From \cref{tab:lgfration}, we could see that the performance is the highest on Flickr under a fraction rate 1\% and on Arxiv under a fraction rate 0.5\%. It is worth noting that in a large graph, when the fraction rate increases, the GNN accuracy also increases, which is different from what we have observed in small graphs such as Pubmed, Cora, and Citeseer. 

\clearpage
\newpage

\bibliography{reference}
\bibliographystyle{ims}

\end{document}